\documentclass[11pt]{article}
\usepackage[margin=1in]{geometry}
\usepackage[authoryear,round]{natbib}
\usepackage[colorlinks=true,linkcolor=blue,citecolor=blue,urlcolor=blue]{hyperref}
\usepackage{graphicx}
\usepackage{amsmath,amssymb,amsthm}
\newtheorem{proposition}{Proposition}[section]
\usepackage{url}
\usepackage{bm}
\usepackage{booktabs}
\usepackage{array}
\usepackage{longtable}
\usepackage[table]{xcolor}
\usepackage[T1]{fontenc}
\usepackage{lmodern}
\usepackage{textcomp}
\usepackage{pgfplots}
\pgfplotsset{compat=1.18}
\usepackage{tikz}
\usetikzlibrary{arrows.meta, decorations.pathreplacing}

\newcommand{\papertitle}{Understanding When Poisson Log-Normal Models Outperform Penalized Poisson Regression for Microbiome Count Data}

\begin{document}

\begin{center}
{\LARGE\bfseries \papertitle\par}
\vspace{1em}
{\large Daniel Agyapong$^{1}$, Julien Chiquet$^{2}$, Jane Marks$^{3}$, Toby Dylan Hocking$^{4}$\par}
\vspace{0.5em}
{\small $^{1}$School of Informatics, Computing, and Cyber Systems, Northern Arizona University, Flagstaff, Arizona, U.S.A.\par}
{\small $^{2}$UMR MIA Paris-Saclay, Universit\'{e} Paris-Saclay, AgroParisTech, INRAE, Palaiseau, France\par}
{\small $^{3}$Department of Biological Sciences, Northern Arizona University, Flagstaff, Arizona, U.S.A.\par}
{\small $^{4}$D\'{e}partement d'informatique, Universit\'{e} de Sherbrooke, Sherbrooke, Qu\'{e}bec, Canada\par}
\vspace{0.4em}
{\small Corresponding author: \texttt{da2343@nau.edu}\par}
\end{center}
\vspace{0.8em}
\begin{abstract}
Multivariate count models are often justified by their ability to capture latent dependence, but researchers receive little guidance on when this added structure improves on simpler penalized marginal Poisson regression.
We study this question using real microbiome data under a unified held-out evaluation framework.
For count prediction, we compare PLN and GLMNet(Poisson) on 20 datasets spanning 32 to 18,270 samples and 24 to 257 taxa, using held-out Poisson deviance under leave-one-taxon-out prediction with 3-fold sample cross-validation rather than synthetic or in-sample criteria.
For network inference, we compare PLNNetwork and GLMNet(Poisson) neighborhood selection on five publicly available datasets with experimentally validated microbial interaction truth.
PLN outperforms GLMNet(Poisson) on most count-prediction datasets, with gains up to 38 percent.
The primary predictor of the winner is the sample-to-taxon ratio, with mean absolute correlation as the strongest secondary signal and overdispersion as an additional predictor.
PLNNetwork performs best on broad undirected interaction benchmarks, whereas GLMNet(Poisson) is better aligned with local or directional effects.
Taken together, these results provide guidance for choosing between latent multivariate count models and penalized Poisson regression in biological count prediction and interaction recovery.
\end{abstract}


\section{Introduction}
Microbial communities profiled by Next-Generation Sequencing produce count data with distinctive statistical properties.
Observations are nonnegative integers, contain many zeros, and are strongly overdispersed.   
Moreover, the number of taxa $D$ routinely exceeds the number of biological samples $N$~\citep{Kodikara2022}.
These properties jointly challenge the assumptions of standard Gaussian models and limit the use of conventional Poisson regression, which cannot accommodate residual correlations among taxa or the extra-Poisson variability that arises from unobserved biological heterogeneity~\citep{mcmurdie2014waste, Chiquet2021}.

The Poisson Log-Normal (PLN) model addresses these limitations by introducing a latent Gaussian layer that captures both overdispersion and dependence among taxa within a principled probabilistic framework~\citep{aitchison1989multivariate, Chiquet2021}.
This latent structure supports a family of downstream analyses, including conditional prediction, dimensionality reduction, and the recovery of sparse ecological interaction networks through graphical lasso regularization~\citep{chiquet2019variational}.
The PLN family is fitted using variational inference and has been increasingly applied in microbiome studies that require explicit modeling of taxon dependencies~\citep{Chiquet2021, SUBEDI20251255}.

Despite this appeal, PLN models carry a practical cost.
Estimating a full $D \times D$ latent covariance matrix scales cubically in $D$, making PLN computationally demanding for the large, sparse count tables typical of modern microbiome surveys.
By contrast, penalized Poisson regression implemented efficiently in \texttt{glmnet}~\citep{Friedman2010} fits each taxon independently using $\ell_1$ regularization, scales near-linearly in $D$, and remains tractable even when $D$ runs into the thousands.
The fundamental question motivating this work is therefore empirical.
Under what conditions do PLN's added modeling complexity translate into measurable gains over penalized Poisson regression?

This question has direct practical consequences.
A researcher analyzing a dataset with $N = 50$ samples and $D = 200$ taxa faces a genuine modeling choice between investing in PLN's joint estimation at computational cost and relying on the faster penalized regression baseline.
Existing literature provides limited guidance.
Prior evaluations of PLN variants have largely focused on illustrative case studies, synthetic data, or narrow analytical tasks, assessing goodness of fit on training data rather than held-out predictive performance~\citep{Chiquet2021, Batardiere2025ZIPLN, Chaussard2025PLNTree}.
As summarized in Table~\ref{tab:related_work}, systematic comparisons against penalized regression baselines on real microbiome data under a unified and reproducible evaluation protocol remain absent from the literature.
Prior work has therefore not provided a held-out comparison built entirely from real microbiome count tables rather than synthetic benchmarks supplemented by a few illustrative applications.
This gap is not merely a matter of coverage: in-sample criteria such as ELBO, BIC, and AIC measure how well a model fits the data it was trained on, not whether its latent structure captures signal that transfers to unseen observations.
A model that overfits the latent covariance can appear superior by in-sample criteria while actually performing worse out of sample.
The absence of held-out evaluation has therefore left the practical question, when does PLN's complexity pay off?

A similar gap exists for network inference.
PLNNetwork recovers ecological interaction networks by estimating a sparse latent precision matrix~\citep{chiquet2019variational}, while neighbourhood selection through \texttt{GLMNet(Poisson)} offers an alternative route to network recovery without a joint model~\citep{Meinshausen2006, Friedman2010}.
To our knowledge, no study has evaluated these competing approaches against experimentally validated microbial interactions, which are necessary for a principled comparison of edge recovery performance.

In this paper, we study the comparative use of joint latent count models and penalized marginal Poisson regression under a common held-out evaluation framework.
Our goal is to provide practical guidance on when the added complexity of a latent multivariate model is justified for count prediction and interaction recovery.
Figure~\ref{fig:pipeline} summarizes the shift from earlier in-sample evaluation practice to the held-out count-prediction protocol used in this study.

The remainder of the paper is organized as follows.
Section~2 reviews related work and situates this study within the existing benchmark literature.
Section~3 describes the algorithms, datasets, and evaluation protocol.
Section~4 presents results for prediction and network inference.
Section~5 provides a discussion of the results, implications, and limitations.
Section~6 concludes.

\section{Related Work}

Microbiome abundance data poses statistical challenges that motivate specialized count models; this section reviews the modelling choices that underlie PLN and the benchmarks that have evaluated it.
Table~\ref{tab:related_work} summarizes how prior studies approach evaluation.

\paragraph{Count data properties of microbiome surveys.}
Microbiome surveys via 16S rRNA amplicon sequencing or shotgun metagenomics produce read-count tables of non-negative integers whose observed values depend on the total number of reads sequenced per sample~\citep{HMP2012, Thompson2017, Xia2023}.
Three properties of these counts make standard Gaussian regression inappropriate, and the latter two also show that plain Poisson regression is insufficient, motivating latent-variable extensions such as the Poisson Log-Normal.
First, \emph{compositionality}: total read depth (library size) varies across samples by orders of magnitude and carries no biological information; models that ignore this offset confound technical sequencing variation with biological signal~\citep{mcmurdie2014waste}.
Second, \emph{overdispersion}: counts exhibit variance far exceeding the Poisson mean, as genuine biological variation across individuals and environments inflates count variance beyond what the Poisson distribution allows~\citep{zhang2017nbmm}. 
Third, \emph{excess zeros}: the majority of entries in the count matrix are zero not because taxa are truly absent but because rare taxa fall below the sequencing detection threshold, with data sparsity commonly exceeding 70\%~\citep{Batardiere2025ZIPLN, Kodikara2022}.
Gaussian regression treats counts as unbounded continuous quantities.
Ad hoc transformations such as $\log(1+x)$ can improve numerical behavior and stability, but they do not by themselves resolve discrete support or compositional structure~\citep{booeshaghi2021log1p}.

\paragraph{The Poisson Log-Normal model.}
The generalized linear model with Poisson or Negative Binomial likelihood is the natural framework for overdispersed count data~\citep{zhang2017nbmm}.
For univariate differential abundance testing, Negative Binomial models are widely used, as implemented in DESeq2~\citep{Love2014DESeq2} and edgeR~\citep{Robinson2010edgeR}.
These are, however, marginal models: each taxon is fitted independently, so the joint dependence structure across taxa is not represented.
The Poisson Log-Normal (PLN) model, first introduced by Aitchison and Ho~\citep{aitchison1989multivariate}, addresses this by placing a multivariate Gaussian distribution on the log-intensities, capturing correlations between different taxa through a shared latent covariance matrix while still respecting the discrete, non-negative nature of the data and accommodating library-size offsets.
This makes PLN a single coherent framework for prediction, ordination, clustering, and network inference from count data~\citep{Chiquet2021, SUBEDI20251255}.

\paragraph{Prediction and representation learning.}
Poisson PCA has been assessed primarily through simulations emphasizing principal component recovery, robustness, and computational efficiency, with microbiome datasets serving as motivating illustrations~\citep{Kenney2021PoissonPCA, Virta2023PoissonPCA}.
These studies evaluate reconstruction error and latent dimension recovery on training data rather than held-out predictive performance, and do not include penalized Poisson regression as a baseline.
The PLN framework paper \citep{Chiquet2021} presents a suite of qualitative applications spanning ordination, regression, network inference, and clustering across several ecological and microbiome datasets, but does not impose a standardized cross-validation protocol or compare against penalized regression.
Extensions addressing zero inflation and hierarchical structure compare PLN variants on synthetic and real microbiome data using in-sample criteria such as variational evidence lower bound (ELBO), log-likelihood or information criteria (AIC/BIC), and downstream discrimination or generative performance~\citep{Batardiere2025ZIPLN, Chaussard2025PLNTree}.

\paragraph{Network inference.}
Network benchmarks for count data graphical models evaluate edge recovery on synthetic PLN graphs and single-cell or gene expression data, reporting ROC curves, AUPR, and edge-weight correlations~\citep{Xiao2022PLNet}.
Synthetic data generators drawing from the PLN distribution have been developed to stress-test network algorithms under controlled graph topologies and sparsity levels~\citep{Qian2024}.
More broadly, \citet{Ghaeli2025} proposed a benchmarking framework focused on reproducibility, using bootstrap resampling to generate modified versions of real microbiome data, but stopped short of validating inferred edges against experimentally confirmed interactions.
Across these studies, edge-recovery assessment relies on synthetic or statistically constructed ground truth.

\begin{table}[htbp]
\caption{Summary of prior benchmarks of PLN and Poisson algorithms.}
\centering
\footnotesize
\setlength{\tabcolsep}{3pt}
\begin{tabular}{>{\raggedright\arraybackslash}p{2.5cm}>{\raggedright\arraybackslash}p{3.2cm}>{\raggedright\arraybackslash}p{3.2cm}>{\raggedright\arraybackslash}p{4.1cm}}
\toprule
Study (year) & Data / setting & Algorithms tested & Evaluation metrics \\
\midrule
\citet{Kenney2021PoissonPCA} & Simulations; microbiome example & Poisson PCA, PLN & Subspace recovery (principal angles); robustness to noise and exposure; runtime \\
\midrule
\citet{Chiquet2021} & Multiple ecology/microbiome case studies & PLN, PLNPCA, PLNNetwork, PLNMixture & BIC/ICL for model order selection; illustrative fits across tasks \\
\midrule
\citet{Xiao2022PLNet} & Synthetic count networks; real scRNA-seq data & PLNet, VPLN, latentcor, glasso & AUPR/AUC for edge recovery; runtime \\
\midrule
\citet{Virta2023PoissonPCA} & Simulations; real count matrices & Poisson PCA, vectorization PCA, matrix-normal methods & Subspace estimation error (principal angles); latent-rank selection accuracy; runtime \\
\midrule
\citet{Qian2024} & Synthetic microbial communities (5 sizes, 4 topologies) & PLN, MIC, Pearson, MA & 3-class accuracy for interaction type recovery; topology and size effects \\
\midrule
\citet{Batardiere2025ZIPLN} & High-zero simulations; cow microbiome & ZIPLN, PLN & In-sample log-likelihood (ELBO/BIC/AIC); latent dispersion; group discrimination \\
\midrule
\citet{Chaussard2025PLNTree} & Synthetic datasets; human gut microbiome & PLN-Tree, mean-field PLN, PLN & ELBO/log-likelihood; Shannon, Simpson, Bray-Curtis diversity; downstream classification accuracy \\
\midrule
This work (2026) & 20 real microbiome datasets; 5 experimental interaction datasets & PLN, PLNNetwork, \texttt{GLMNet(Poisson)}, featureless & Prediction: held-out Poisson deviance under 3-fold LOTO-CV; network inference: F1 score against experimental ground truth; runtime; memory \\
\bottomrule
\end{tabular}
\label{tab:related_work}
\end{table}

\paragraph{Summary of contributions.}
Existing benchmarks of Poisson regression and Poisson Log Normal algorithms have primarily targeted narrow analytical tasks, evaluating individual variants such as Poisson PCA~\citep{Kenney2021PoissonPCA, Virta2023PoissonPCA}, PLN regression or network formulations~\citep{Chiquet2021, Xiao2022PLNet}, and zero-inflated or hierarchical extensions~\citep{Batardiere2025ZIPLN, Chaussard2025PLNTree} but typically in isolation, on synthetic data, or through qualitative case studies rather than unified systematic comparisons.
Most prior evaluations emphasize synthetic or illustrative demonstrations, lack standardized cross-validation protocols, omit uncertainty quantification, and rarely compare PLN variants directly to penalized Poisson regression algorithms on real microbiome data. 
As shown in Table~\ref{tab:related_work}, the field still lacks a unified empirical framework capable of assessing both predictive accuracy and model stability under consistent experimental conditions.

\begin{itemize}
  \item We introduce LOTO-CV as a held-out evaluation framework for Poisson count models, filling a methodological gap in the PLN benchmark literature where all prior evaluations use in-sample criteria. Applied to 20 real microbiome datasets, this yields the first out-of-sample comparison of PLN and \texttt{GLMNet(Poisson)} on real data.
  \item We identify $N/D$ as the primary empirical predictor of the dataset-level winner, MAC as the strongest secondary signal, and overdispersion as an additional predictor. PLN wins in 12 of 14 datasets with $N/D < 5$ and in none of the 6 with $N/D \geq 5$, improving on \texttt{GLMNet(Poisson)} by up to 38\% in the favorable regime.
  \item We evaluate PLNNetwork and \texttt{GLMNet} neighbourhood selection against all five publicly available datasets with experimentally validated microbial interaction truth. This is the first comparison of these approaches on real biological ground truth rather than simulated networks.
\end{itemize}

\begin{figure}[!t]
\centering
\includegraphics[width=0.92\textwidth]{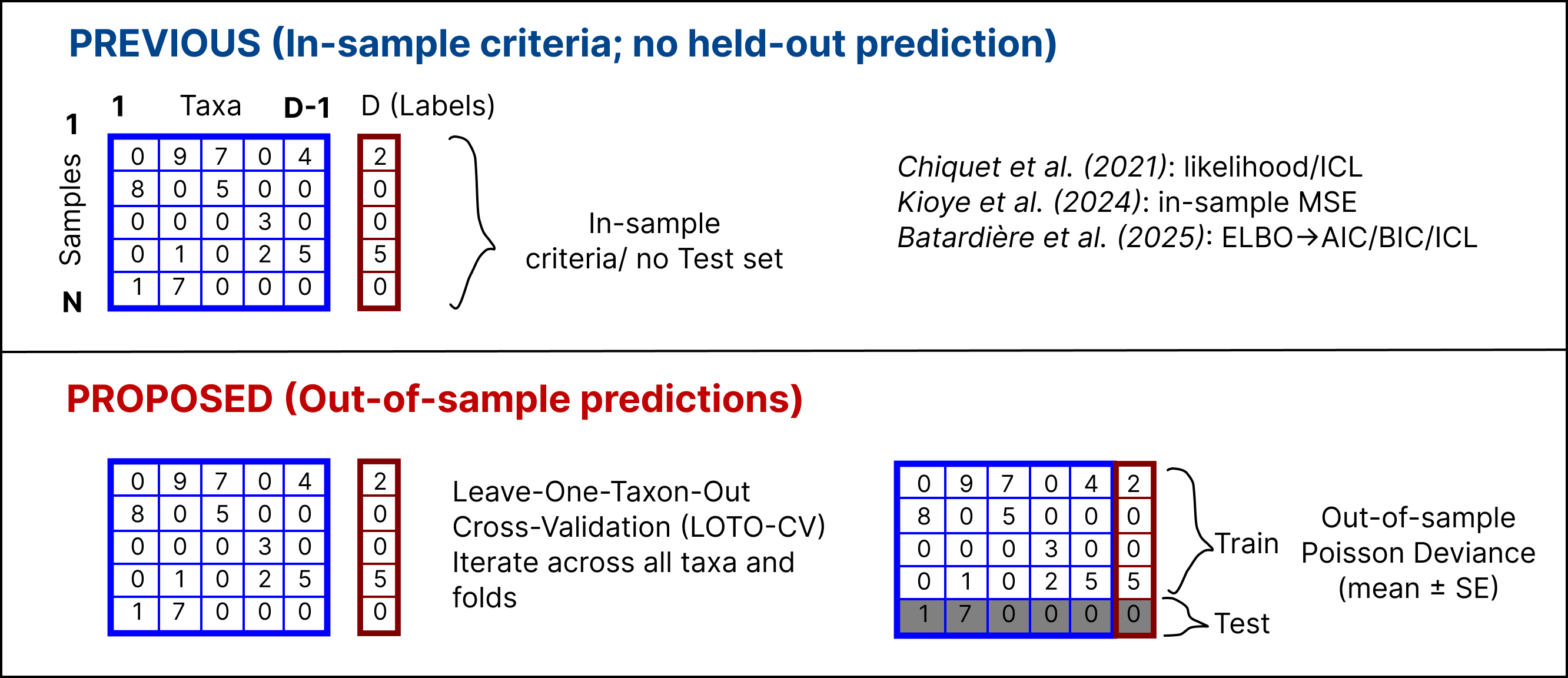}
\caption{Conceptual contrast between earlier evaluation practice and the count-prediction protocol used in this study.
Previous work often assessed count models with in-sample fit criteria computed on the observed abundance matrix.
Our benchmark instead withholds a target taxon, predicts its counts out of sample from the remaining taxa, and scores those predictions with held-out Poisson deviance under leave-one-taxon-out cross-validation.}
\label{fig:pipeline}
\end{figure}

\section{Methods}
\subsection{Datasets}
\subsubsection{Count-prediction datasets}
\label{subsec:count-prediction-datasets}
We evaluated count prediction on a heterogeneous benchmark of 20 publicly available real microbiome datasets spanning oral, skin, nasal, stool, infant, pediatric, and general gut systems, together with several collections focused on disease including colorectal (CRC), \textit{Clostridium difficile} infection, HIV, obesity, and type~1 diabetes.
The retained benchmark set includes feature tables at the family, genus, species, and OTU levels and covers a broad range of sample sizes and feature dimensions.
Across datasets, the number of samples ranged from $N=32$ to $N=18{,}270$, the number of features ranged from $D=24$ to $D=257$, and the fraction of zero entries in the benchmarked abundance matrices ranged from $30.64\%$ to $93.01\%$.

The goal of this collection is not to represent a single biological system, but to cover the range of settings in which researchers actually face a choice between a joint latent count model and a faster penalized regression baseline.
Several dataset families are particularly important for that comparison.
American Gut 2~\citep{McDonald2018}, MehtaRS 2018~\citep{Mehta2018}, and MBQC integrated OTUs~\citep{MBQC2017} represent larger and more heterogeneous microbiome surveys, including a very large study in the MBQC collection.
These datasets are useful because they test whether PLN retains an advantage when sample size is no longer the main bottleneck.

Other datasets probe the opposite end of the benchmark.
Lozupone HIV~\citep{Lozupone2013} and several of the disease-focused stool datasets have relatively small $N$, moderate or large $D$, or both.
These are the settings in which the balance between latent covariance estimation and penalized univariate fitting becomes most visible.
The benchmark also includes multiple colorectal cancer datasets~\citep{Wang2012, Zeller2014}, early-life gut studies such as Diabimmune Karelia~\citep{Vatanen2016} and ShaoY 2019~\citep{Shao2019}, and targeted disease collections such as Schubert CDI~\citep{Schubert2014}, which together provide repeated examples of related biological questions under different sample sizes, taxonomic resolutions, and sparsity levels.

This diversity is important for interpretation.
The paper's goal is not to identify a single best-performing method on one cohort, but to determine which properties of the count matrix predict when PLN is worth its added complexity.
The broad spread of $N$, $D$, sparsity, and biological settings is therefore a deliberate part of the benchmark design rather than background variation to be averaged away.
Full dataset metadata, including the representative inventory are provided in Table~S2 of the Supplementary Material.

\subsubsection{Network-inference datasets}
\label{subsubsec:network-inference-datasets}
We evaluated network inference on five real microbial interaction benchmarks with experimentally supported edge truth.
These datasets were chosen because each provides an external experimental source of interaction truth rather than a simulated or statistically constructed target.
They therefore allow direct evaluation of edge recovery on biological systems for which pairwise competition, facilitation, or conditional effects were measured independently of the fitted models.

The benchmark contains two distinct kinds of truth.
The first kind targets broad undirected edge recovery.
OMM12~\citep{Weiss2022OMM12} and OMM12 keystone 2023~\citep{Weiss2023Keystone} are defined human gut communities in which community composition was perturbed through coculture, deletion, or keystone experiments.
These datasets are well suited to evaluating recovery of an overall community interaction graph because the truth reflects system-wide effects across a small, fully enumerated set of taxa.
PairInteraX~\citep{Zhu2025PairInteraX} plays a complementary role.
It provides a much larger human gut species panel with experimentally measured pairwise coculture interaction labels, which makes it useful for testing whether the same methods still recover true edges when the number of taxa increases substantially.

The second kind targets local or directional effects.
Butyrate assembly 2021~\citep{Clark2021Butyrate} contains singleton, pair, and larger synthetic gut communities, allowing local effects to be defined from pair-vs-singleton abundance shifts.
Host fitness 2018~\citep{Gould2018HostFitness} is a five-species \textit{Drosophila} gut system in which pair-vs-singleton colony counts were linked to host reproductive fitness.
These two datasets are informative because the underlying signal is closer to ``taxon $a$ affects taxon $b$'' than to a single global undirected graph, and thus test whether nodewise and joint graphical procedures behave differently when the truth is more local.

Across these datasets, sample size ranged from $N=109$ to $N=1449$, feature dimension ranged from $D=5$ to $D=96$, and data sparsity ranged from $21.25\%$ to $66.19\%$.
Processed abundance matrices and truth edge sets were harmonized into a common benchmark format so that PLNNetwork and \texttt{GLMNet(Poisson)} could be evaluated under the same scoring rules despite their different biological origins.
Dataset metadata are provided in Table~S3 of the Supplementary Material.

\subsection{Competing methods}
\label{subsec:algorithms}

We used different methods for count prediction and network inference.
Let $Y_{ij}$ denote the observed count for sample $i$ and taxon $j$, and let $s_i$ denote the sequencing-depth offset for sample $i$.
For count prediction, we compared a featureless baseline, \texttt{GLMNet(Poisson)}, and PLN.
For network inference, we compared \texttt{GLMNet(Poisson)} neighbourhood selection, and PLNNetwork.

For count prediction, the featureless baseline predicts each target taxon by the training-set mean
\[
\hat{\mu}_{ij} = \bar{Y}_{\cdot j}.
\]
For network inference, the diagonal baseline predicts an empty graph, that is
\[
\hat{A}_{jk} = 0 \qquad \text{for all } j \neq k.
\]

\texttt{GLMNet(Poisson)} fits a penalized Poisson regression for each target taxon~\citep{Friedman2010}.
For count prediction, we model
\[
Y_{ij} \mid Y_{i,-j} \sim \mathrm{Poisson}(\mu_{ij}),
\qquad
\log \mu_{ij} = \log s_i + \beta_{0j} + \sum_{k \neq j} \beta_{jk}\,\log(1 + Y_{ik}),
\]
with coefficients estimated by minimizing
\[
-\ell_j(\beta_{0j}, \beta_j) + \lambda_j \lVert \beta_j \rVert_1.
\]
This yields a separate sparse Poisson regression for each target taxon and does not estimate a joint latent covariance across taxa.
For network inference, we use nodewise neighbourhood selection~\citep{Meinshausen2006}.
An undirected edge is included when the symmetrized coefficient support is nonzero, that is
\[
\hat{A}_{jk} = \mathbb{1}\!\left\{\max\!\big(|\hat{\beta}_{jk}|, |\hat{\beta}_{kj}|\big) > 0\right\}.
\]

PLN fits a Poisson log-normal model with a latent Gaussian layer~\citep{aitchison1989multivariate,Chiquet2021}.
For each sample,
\[
Y_{ij} \mid Z_i \sim \mathrm{Poisson}\!\left(s_i \exp(\eta_j + Z_{ij})\right),
\qquad
Z_i \sim \mathcal{N}(0, \Sigma).
\]
For count prediction, held-out means are obtained from Gaussian conditioning in the latent layer,
\[
\hat{\mu}_{ij}
=
s_i \exp\!\left(\hat{\eta}_j + \hat{m}_{ij \mid -j} + \tfrac{1}{2}\hat{v}_{j \mid -j}\right),
\]
where $\hat{m}_{ij \mid -j}$ and $\hat{v}_{j \mid -j}$ are the conditional latent mean and variance for taxon $j$ given the observed taxa in sample $i$.
To stabilize these predictions, we applied covariance shrinkage with parameter $\alpha \in \{0.0, 0.1, \ldots, 1.0\}$ by scaling the off-diagonal entries of $\hat{\Sigma}$.

PLNNetwork extends PLN by estimating a sparse precision matrix $\Omega = \Sigma^{-1}$ within the same variational framework~\citep{chiquet2019variational,Chiquet2021}.
Conceptually, the fitted model maximizes a penalized variational objective of the form
\[
\mathrm{ELBO}(\eta, \Omega) - \rho \lVert \Omega \rVert_{1,\mathrm{off}},
\]
so that zero off-diagonal entries of $\Omega$ correspond to absent conditional associations.
For network inference, we initialized PLNNetwork from a full PLN fit, evaluated a matched log-spaced penalty grid from $\rho_{\max}$ to $0.005\,\rho_{\max}$, set diagonal penalization to false, and selected the final graph by EBIC.
The estimated interaction graph is therefore
\[
\hat{A}_{jk} = \mathbb{1}\!\left\{\hat{\Omega}_{jk} \neq 0\right\}.
\]

All PLN models were fit with variational inference using the \texttt{PLNmodels} package~\citep{Chiquet2021}.
All penalized Poisson regressions were fit with \texttt{glmnet}~\citep{Friedman2010}.

\subsection{Evaluation protocols}
\label{subsec:evaluation-protocols}

\subsubsection{Count-prediction evaluation}

For count prediction, we used a leave-one-taxon-out prediction protocol combined with 3-fold cross-validation over samples.
Each taxon is treated in turn as the prediction target, the remaining taxa serve as predictors, and predictive performance is evaluated on held-out samples.
The resulting held-out scores are then averaged across target taxa and sample folds.
The inner CV used 3 folds for hyperparameter selection.
The primary metric is mean held-out Poisson deviance, defined as
\[
  \mathcal{D}(y, \hat\mu) = \frac{2}{m}\sum_{i=1}^{m} \left( y_i \log\frac{y_i}{\hat\mu_i} - (y_i - \hat\mu_i) \right),
\]
where $m$ is the number of held-out samples, $y_i$ are the observed counts, and $\hat\mu_i$ are the predicted means.
We use the standard convention $0 \log 0 = 0$.
We used Poisson deviance rather than RMSE, correlation, or absolute-error criteria because all competing models are fitted under Poisson mean assumptions.
Poisson deviance is a proper scoring rule for count data: it is minimised in expectation by the true conditional mean $\mathbb{E}[Y_i \mid x_i]$, so a model that achieves lower held-out deviance genuinely predicts better rather than simply fitting the training data more closely~\citep{Gneiting2007ProperScoring}.
This property, which in-sample criteria such as ELBO and BIC do not share, is what makes held-out Poisson deviance the principled choice for comparing count models out of sample.
Gaussian losses such as RMSE are less appropriate here because they ignore the mean-variance relationship that is central to the comparison.
Formal definitions, a proof of properness, and a discussion of the exchangeability assumption underlying the estimator are given in Appendix~S1.
For \texttt{GLMNet(Poisson)}, the regularization parameter $\lambda$ was selected by inner-fold Poisson deviance using \texttt{cv.glmnet}.
For PLN, the shrinkage parameter $\alpha$ was selected by 3-fold inner cross-validation on Poisson deviance.
The featureless baseline serves as a lower bound: any algorithm with higher deviance than the featureless baseline provides no useful predictive signal.

\subsubsection{Network-inference evaluation}

For network inference, estimated interaction graphs were compared against experimentally validated microbial interactions from five ground-truth datasets: PairInteraX~\citep{Zhu2025PairInteraX}, OMM12~\citep{Weiss2022OMM12}, OMM12 keystone 2023~\citep{Weiss2023Keystone}, butyrate assembly 2021~\citep{Clark2021Butyrate}, and host fitness 2018~\citep{Gould2018HostFitness}.
Each dataset provides experimentally derived pairwise interaction labels obtained from co-culture, dropout, or combinatorial abundance experiments.
The primary metric is the F1 score on edge recovery, computed by matching predicted edges against the ground-truth interaction set.
For PLNNetwork, the penalty parameter was selected by EBIC on the training data.
For \texttt{GLMNet(Poisson)} neighbourhood selection, $\lambda$ was fixed at \texttt{lambda.1se} from \texttt{cv.glmnet}.
The diagonal baseline, which predicts an empty graph, serves as the trivial reference point.

\subsection{Computational setup}
\label{subsec:computational-setup}

All experiments were run on Northern Arizona University's Monsoon high-performance computing cluster~\citep{Buechler2025MonsoonIntro}.
Prediction and network benchmark jobs were executed as single-core Slurm jobs.

Runtime and memory scaling were profiled using the \texttt{atime} R package on the Vogtmann et al.\ stool species dataset~\citep{Vogtmann2016CRC}.
For the saved profiling run, the number of taxa $D$ was varied from 10 to 1000 at a fixed sample size of $N=110$, and each configuration was run three times.
We report median wall time and peak memory across those repeated measurements.

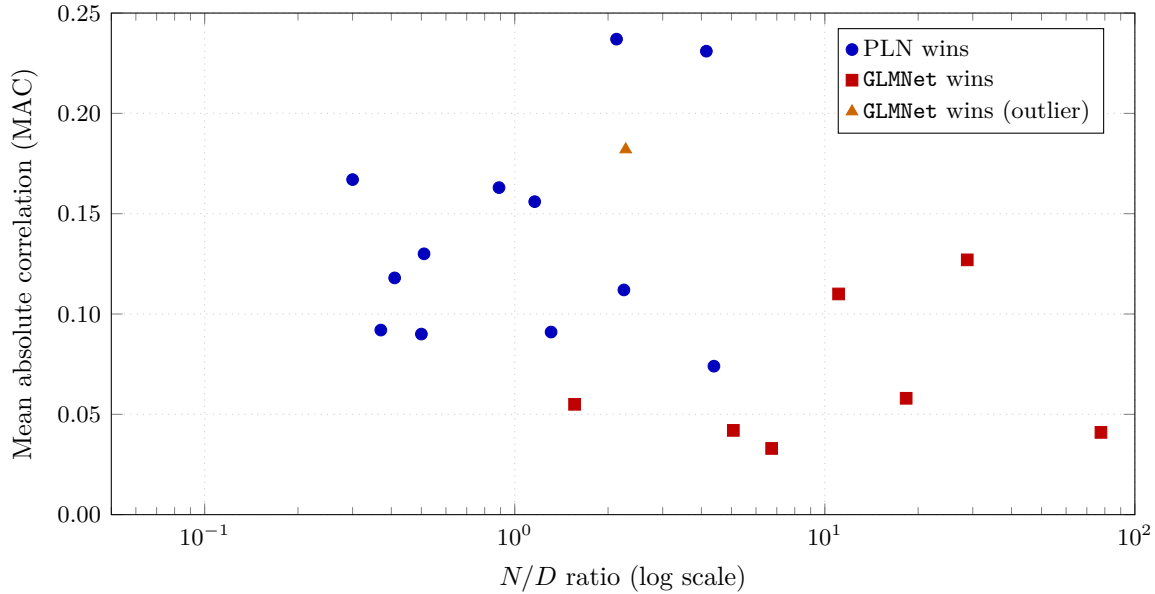
\begin{figure}[!t]
\centering
\begin{tikzpicture}
\begin{axis}[
  width=0.92\textwidth,
  height=0.50\textwidth,
  xlabel={$N/D$ ratio (log scale)},
  ylabel={Mean absolute correlation (MAC)},
  xmode=log,
  xmin=0.05, xmax=100,
  ymin=0, ymax=0.25,
  ytick={0, 0.05, 0.10, 0.15, 0.20, 0.25},
  yticklabel style={/pgf/number format/fixed, /pgf/number format/fixed zerofill,
                    /pgf/number format/precision=2, font=\footnotesize},
  ymajorgrids=true,
  xmajorgrids=true,
  grid style={dotted, gray!40},
  tick label style={font=\footnotesize},
  label style={font=\small},
  legend style={font=\footnotesize, at={(0.97,0.97)}, anchor=north east},
  legend cell align=left,
]
\addplot[only marks, mark=*, mark size=2.2pt, blue!75!black]
  coordinates {
    (0.30,0.167)(0.37,0.092)(0.41,0.118)(0.50,0.090)
    (0.51,0.130)(0.89,0.163)(1.16,0.156)(1.31,0.091)
    (2.13,0.237)(2.25,0.112)(4.15,0.231)(4.39,0.074)
  };
\addlegendentry{PLN wins}
\addplot[only marks, mark=square*, mark size=2.2pt, red!75!black]
  coordinates {
    (1.56,0.055)(5.07,0.042)(6.74,0.033)(11.08,0.110)
    (18.30,0.058)(28.80,0.127)(77.74,0.041)
  };
\addlegendentry{\texttt{GLMNet} wins}
\addplot[only marks, mark=triangle*, mark size=2.5pt, orange!80!black]
  coordinates {
    (2.28,0.182)
  };
\addlegendentry{\texttt{GLMNet} wins (outlier)}
\end{axis}
\end{tikzpicture}
\caption{Dataset-level winner as a function of $N/D$ and MAC across the 20-dataset real-data benchmark.
Blue circles: PLN achieves lower held-out Poisson deviance.
Red squares: \texttt{GLMNet(Poisson)} achieves lower held-out Poisson deviance.
The orange triangle marks American Gut~1, a GLMNet win at low $N/D$ and high MAC that runs counter to the general pattern.}
\label{fig:nd_mac_scatter}
\end{figure}

\section{Results}
\subsection{Count-prediction results}

The central empirical result is that relative predictive performance varies systematically across datasets rather than favoring a single method.
Figure~\ref{fig:nd_mac_scatter} shows that the sample-to-taxon ratio $N/D$ is the primary determinant of outperformance across datasets. 
PLN wins in 12 of 14 datasets with $N/D < 5$ and in none of the 6 datasets with $N/D \geq 5$.
PLN tends to win when $N/D$ is small and the average dependence among taxa is strong.
\texttt{GLMNet(Poisson)} tends to win when $N/D$ is large and inter-taxon dependence is weaker.

Mean Absolute Correlation (MAC) is the strongest secondary signal and provides a visually intuitive summary of inter-taxon dependence strength.
Among datasets where \texttt{GLMNet(Poisson)} wins despite low $N/D$, CosteaPI~2017 ($N/D = 1.56$, MAC $= 0.055$) has distinctly low MAC, consistent with weak community dependence reducing the benefit of PLN's latent layer.
American Gut~1 ($N/D = 2.28$, MAC $= 0.182$) is an outlier where \texttt{GLMNet(Poisson)} wins despite moderate dependence.
Overdispersion is an additional predictor that further distinguishes settings where PLN's latent Gaussian layer is most beneficial.
This aligns with the structure of the PLN model.
Its latent Gaussian layer is most useful when counts depart substantially from the Poisson assumption and when joint dependence among taxa carries predictive signal.

Table~S4 reports the complete retained 20-dataset count-prediction benchmark and highlights the strongest dataset-level contrasts in both directions.
The table includes standard errors, paired $p$-values from a Wilcoxon signed-rank test, and Benjamini--Hochberg adjusted $q$-values to account for multiple testing across the 20 datasets.

\begin{figure}[!t]
  \centering
  \includegraphics[width=0.95\textwidth]{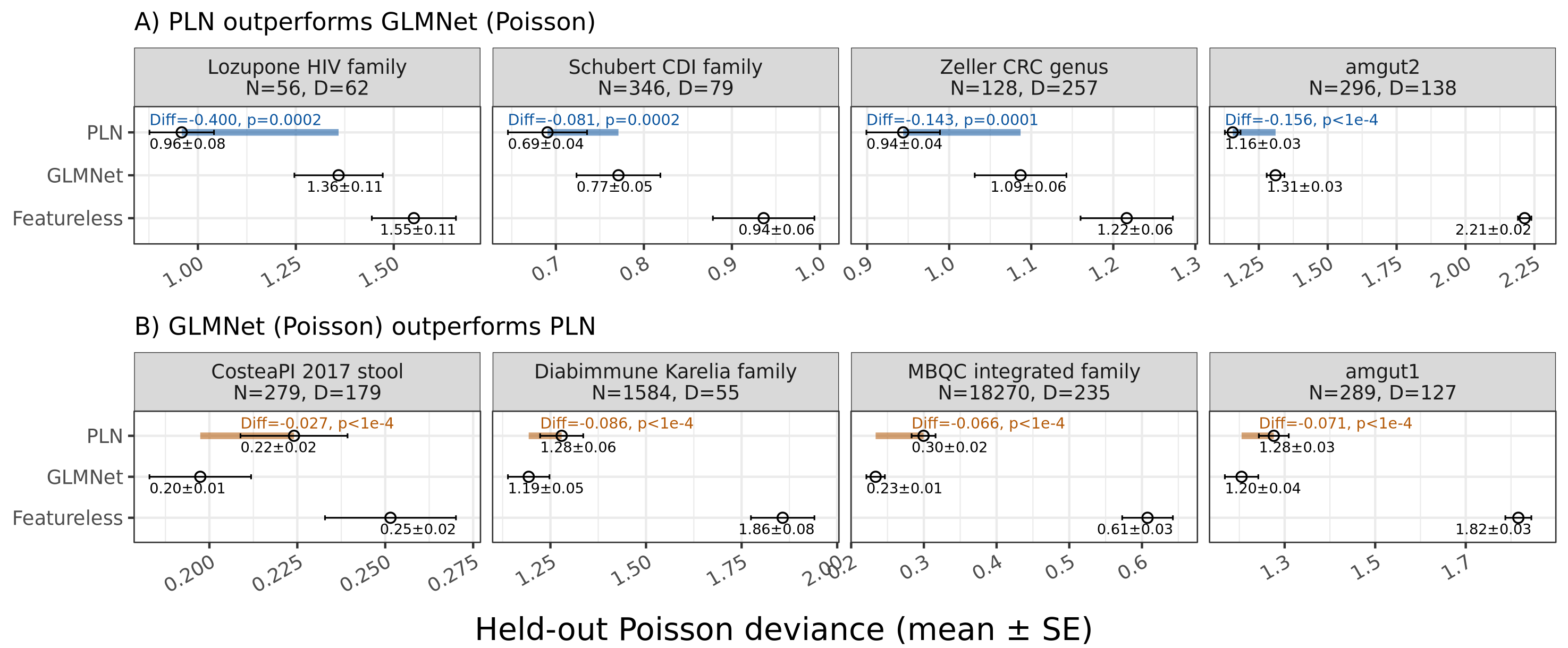}
  \caption{Representative datasets illustrating both directions of the PLN
  vs.\ GLMNet(Poisson) comparison.
  Top: datasets where PLN achieves lower held-out Poisson deviance.
  Bottom: datasets where GLMNet(Poisson) achieves lower held-out Poisson deviance.
  Each panel reports held-out Poisson deviance (mean $\pm$ SE across three outer
  cross-validation folds) for PLN, GLMNet(Poisson), and the featureless baseline.
  The annotated difference and $p$-value are from the paired comparison of PLN
  and GLMNet(Poisson) across folds.}
  \label{fig:count_pred_case_study}
\end{figure}

Figure~\ref{fig:count_pred_case_study} makes the same contrast concrete with representative case studies from both parts of the benchmark.
In the upper panel, Lozupone HIV family, Zeller CRC genus, Schubert CDI family, and amgut2 all show clear reductions in held-out Poisson deviance relative to \texttt{GLMNet(Poisson)}.
These are settings in which limited effective sample size or stronger dependence among taxa makes joint latent modeling useful.

The lower panel shows the opposite pattern.
CosteaPI 2017 stool, MBQC integrated, Diabimmune Karelia, and amgut1 all favor \texttt{GLMNet(Poisson)}, with the largest advantages appearing in better sampled settings where independent penalized regressions are already sufficient.
In all eight examples, both fitted methods improve over the featureless baseline, indicating that the comparison is between two informative models rather than between signal and noise.

These results suggest a simple practical rule.
PLN is most attractive when $N/D$ is small, inter-taxon dependence is strong, and counts are overdispersed, because those are the settings in which latent covariance estimation appears to improve prediction.
\texttt{GLMNet(Poisson)} becomes preferable as $N/D$ grows and counts approach the Poisson limit, where the extra flexibility of PLN adds cost without commensurate gains.
These are empirical heuristics rather than hard thresholds, but they provide a more concrete starting point for method selection than has previously been available.

\subsection{Network-inference results}
\label{subsec:network-inference-results}

\begin{figure}[!t]
  \centering
  \includegraphics[width=0.94\textwidth]{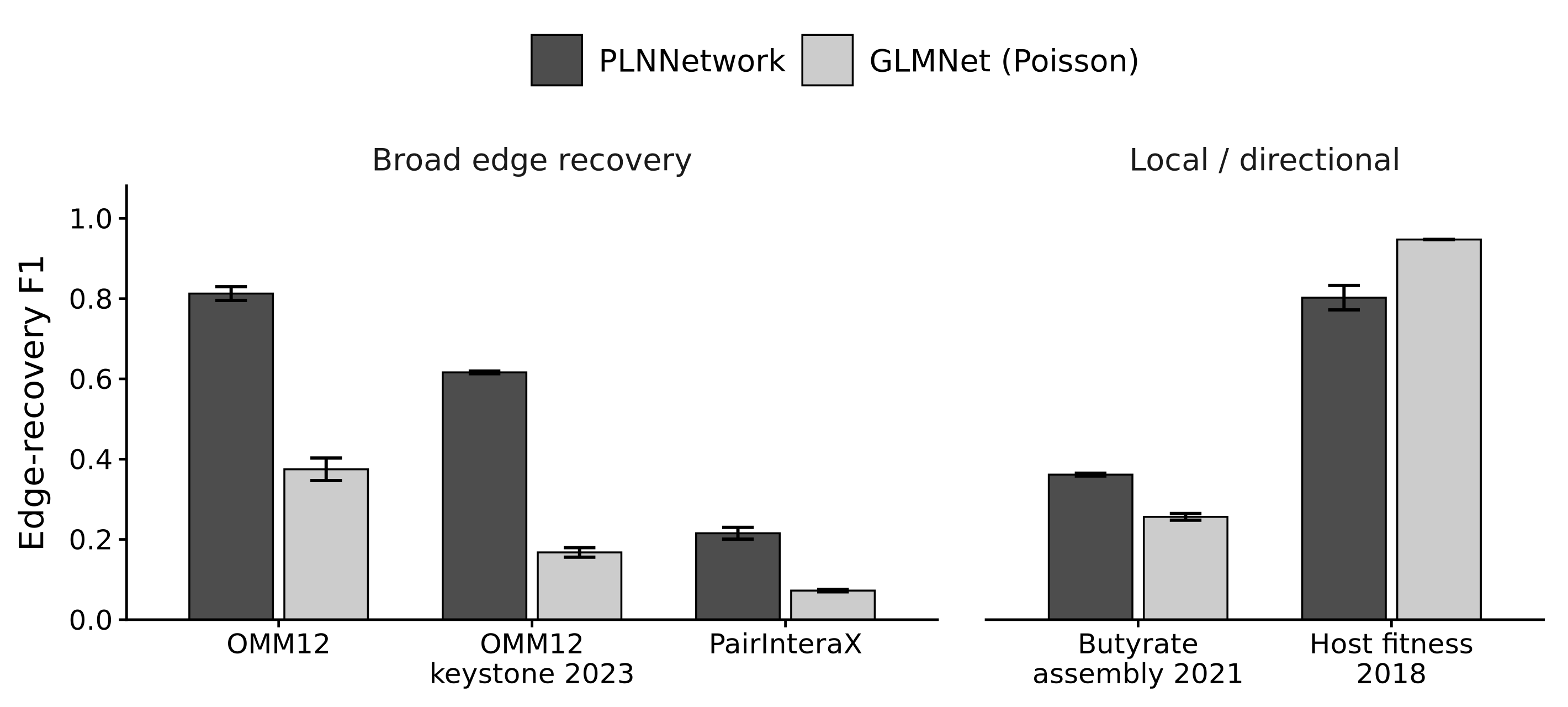}
  \caption{Edge-recovery F1 score on five experimental ground-truth datasets for
  PLNNetwork and \texttt{GLMNet(Poisson)}.
  Datasets are grouped by ground-truth type: broad edge recovery (OMM12,
  OMM12 keystone 2023, PairInteraX) and local/directional (butyrate assembly 2021,
  host fitness 2018).
  Bars show mean F1 score over $B = 20$ bootstrap resamples of the abundance matrix;
  error bars are $\pm 1$ SE.}
  \label{fig:network_f1}
\end{figure}

PLNNetwork outperforms \texttt{GLMNet(Poisson)} on four of the five real interaction-truth datasets, but the margin depends strongly on the type of biological truth being evaluated (Figure~\ref{fig:network_f1}).
On the three broad edge-recovery datasets, PLNNetwork has the clearer advantage.
Its mean bootstrap F1 score exceeds that of \texttt{GLMNet(Poisson)} on OMM12 ($0.81 \pm 0.02$ versus $0.37 \pm 0.03$), OMM12 keystone 2023 ($0.62 \pm 0.003$ versus $0.17 \pm 0.01$), and PairInteraX ($0.22 \pm 0.01$ versus $0.07 \pm 0.003$).
These datasets reward recovery of a global undirected interaction structure, which matches the inductive bias of a sparse latent precision matrix.

The OMM12 example in Figure~\ref{fig:network_case_omm12} illustrates this contrast.
OMM12 is a defined gut community comprising 12 bacterial taxa shown as nodes in Figure~\ref{fig:network_case_omm12}.
The experimental graph is relatively dense and dominated by negative interactions.
PLNNetwork recovers much of that structure, whereas \texttt{GLMNet(Poisson)} returns a substantially sparser graph with a different balance of edge signs.

The local and directional datasets tell a more nuanced story.
On butyrate assembly 2021, PLNNetwork retains a modest advantage ($0.36 \pm 0.004$ versus $0.26 \pm 0.008$).
On host fitness 2018, however, the ranking reverses and \texttt{GLMNet(Poisson)} achieves the best overall recovery ($0.95 \pm 0.00$ versus $0.80 \pm 0.03$).
This contrast is shown in Figure~S2 of the Supplementary Material.
The five taxa are \textit{Lactobacillus plantarum} (LP), \textit{Lactobacillus brevis} (LB), \textit{Acetobacter pasteurianus} (AP), \textit{Acetobacter tropicalis} (AT), and \textit{Acetobacter orientalis} (AO).
PLNNetwork recovers part of the experimentally supported structure but misses edges involving AO and introduces a positive LP--LB edge that is not present in the truth.
By contrast, \texttt{GLMNet(Poisson)} closely matches the local negative interaction pattern and misses only one true edge.
In this setting, the ground truth is driven by targeted pairwise effects rather than by a broad undirected community network, and nodewise Poisson regressions appear better aligned with that local signal.

\begin{figure}[!t]
  \centering
  \includegraphics[width=0.94\textwidth]{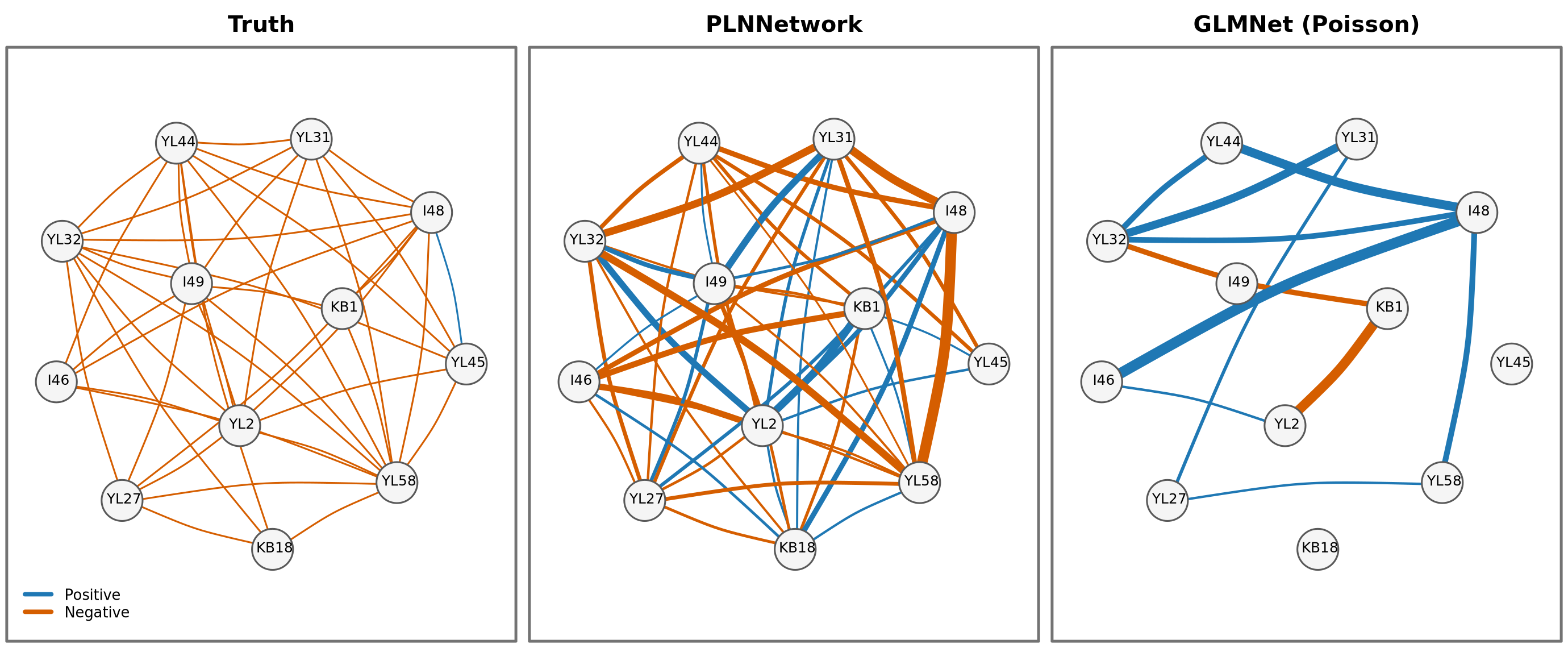}
  \caption{Ground-truth and inferred networks for the OMM12 defined community (12 taxa).
  Each panel shows the experimental ground truth (left), the PLNNetwork prediction
  (centre), and the \texttt{GLMNet(Poisson)} prediction (right).
  Node labels are strain codes: KB1 (\textit{E.\ faecalis}), YL2 (\textit{B.\ animalis}), KB18 (\textit{A.\ muris}), YL27 (\textit{M.\ intestinale}), YL31 (\textit{F.\ plautii}), YL32 (\textit{E.\ clostridioformis}), YL44 (\textit{A.\ muciniphila}), YL45 (\textit{T.\ muris}), I46 (\textit{C.\ innocuum}), I48 (\textit{B.\ caecimuris}), I49 (\textit{L.\ reuteri}), YL58 (\textit{B.\ coccoides}).
  Edge colour indicates interaction sign (blue: positive, orange: negative).}
  \label{fig:network_case_omm12}
\end{figure}

\subsection{Computational scaling}
\label{subsec:computational-scaling}

Computational cost differs sharply between the two model classes.
Across the full range of $D$ values in Figure~\ref{fig:atime}, PLN is consistently slower and more memory-intensive than \texttt{GLMNet(Poisson)}.
At $D = 1000$, PLN requires about 76\,s of wall time and 392\,MB of peak memory, whereas \texttt{GLMNet(Poisson)} requires about 0.28\,s and 79\,MB.

The scaling pattern matches the model structure.
PLN must estimate and manipulate a dense latent covariance matrix, so its cost increases much more quickly with dimension.
\texttt{GLMNet(Poisson)} fits separate penalized regressions and remains comparatively cheap even at the largest tested dimension.
Even so, the observed PLN costs remain manageable on the benchmark sizes studied here, which makes the practical trade-off one of speed versus accuracy rather than feasibility versus infeasibility.

\begin{figure}[!t]
  \centering
  \includegraphics[width=0.94\textwidth]{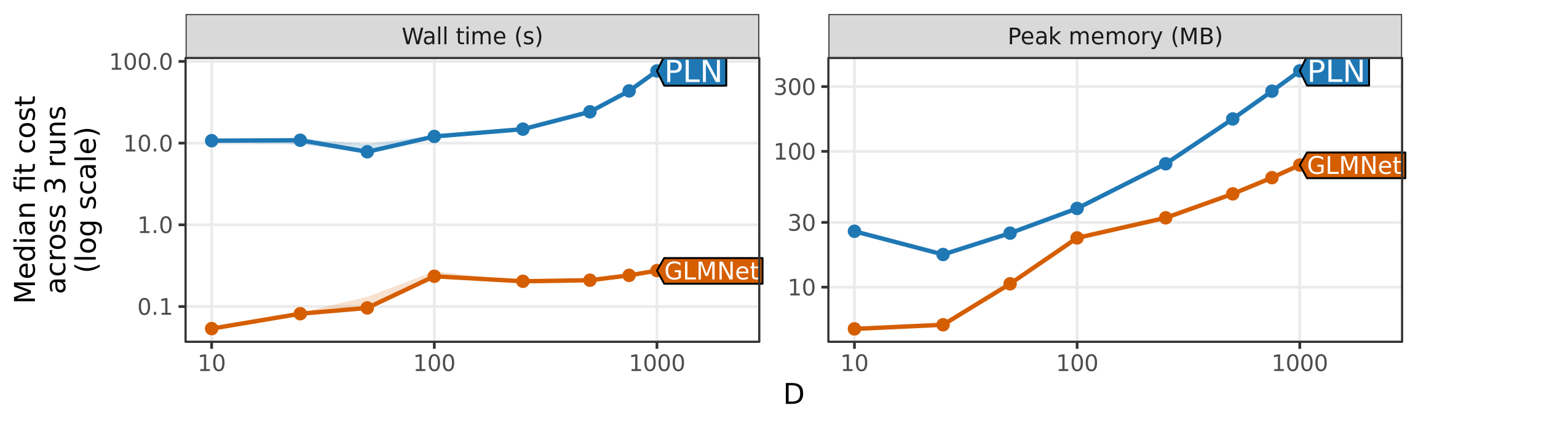}
  \caption{Median wall time (left) and peak memory (right) as a function of the
  number of taxa $D$, profiled on the Vogtmann et al.\ stool species dataset
  ($N = 110$) using the \texttt{atime} R package.
  Each point is the median of three runs; $D$ ranges from 10 to 1000.
  Both axes are on a log scale.}
  \label{fig:atime}
\end{figure}

\section{Discussion}
\label{sec:discussion}

The count-prediction results provide direct comparative guidance on when a joint latent count model is worth its additional complexity.
PLN must estimate a latent covariance structure with $O(D^2)$ degrees of freedom, so its advantage depends on whether the data contain enough information to support that extra layer.
When $N/D$ is small, inter-taxon dependence is strong, and counts are overdispersed, the latent Gaussian component captures signal that independent Poisson regressions miss.
When $N/D$ is large or inter-taxon dependence is weak, the extra flexibility of PLN adds variance and computational cost without improving prediction.
The win-rate pattern is direct: PLN takes 12 of 14 datasets below $N/D = 5$ and zero of the 6 above it, with MAC and overdispersion providing additional explanatory power for the residual cases.

The datasets where PLN shows the strongest gains illustrate the biological environments in which latent dependence modeling is most useful.
Castro-Nallar~2015 (oral cavity, $N/D = 0.30$, $+27\%$) and Lozupone~HIV ($N/D = 0.89$, $+30\%$) are small clinical studies from body sites and disease states known for tight bacterial co-associations, where the joint latent structure captures signal that independent regressions miss.
Baker~NASH/obesity ($N/D = 1.16$, $+30\%$) and Ross~obesity ($N/D = 0.51$, $+30\%$) involve metabolically disrupted gut communities where co-abundance patterns among taxa are elevated.
By contrast, PLN's largest losses occur in large population surveys: MBQC ($N/D = 77.74$, $-28\%$) and MehtaRS~2018 ($N/D = 5.07$, $-25\%$) are high-throughput multi-site collections where samples are abundant relative to taxa and the benefit of modeling joint dependence vanishes.
This suggests PLN is most valuable when studying disrupted or anatomically distinct communities under resource-constrained sampling, and that \texttt{GLMNet(Poisson)} suffices when large healthy cohorts are available.

The network-inference results point to a similar match between model structure and evaluation target.
PLNNetwork estimates a sparse undirected precision matrix and therefore aligns naturally with broad community-level interaction structure.
\texttt{GLMNet(Poisson)} uses nodewise regressions and is correspondingly better suited to local or directional effects.
The host fitness benchmark is the clearest example of this distinction, because its truth labels reflect targeted pairwise effects rather than a single global undirected graph.

Several limitations should be kept in mind.
The 20 real microbiome count prediction datasets are not fully independent: the benchmark includes multiple colorectal cancer cohorts and multiple HMP body-site collections, so the effective number of independent biological comparisons is somewhat less than 20.
The win-rate pattern is nonetheless consistent across biological environments, which limits the concern, but the findings should not be interpreted as arising from 20 entirely unrelated experiments.
The comparison is restricted to the Poisson model family by design; SPIEC-EASI~\citep{Kurtz2015SPIECEASI} and SparCC~\citep{Friedman2012SparCC} operate under compositional log-ratio assumptions that differ fundamentally from the Poisson likelihood, and including them would conflate model family with evaluation protocol.
The network benchmark uses all five experimental ground-truth datasets with publicly available microbial interaction truth that we are aware of; the scarcity of such data reflects fundamental experimental constraints rather than a selective choice, and the conclusions should be interpreted in light of this limitation.
Both fitted methods also receive the same log1p-transformed inputs, a design choice that improves comparability but may not be optimal for PLN.
Finally, the computational profiling is based on a single dataset and single-core execution, so it should be interpreted as a controlled comparison rather than as an exhaustive systems benchmark, although the observed costs are consistent with the expected scaling patterns and with prior reports of PLN runtimes.

\section{Conclusion}
\label{sec:conclusion}

This study provides a comparative evaluation of joint latent and penalized marginal Poisson models for microbiome count prediction and interaction recovery.
Across 20 real count microbiome datasets and five real interaction-truth benchmarks, neither model class dominates universally.
Instead, the benchmark identifies the conditions under which each approach is most appropriate.
For count prediction, PLN is most useful when samples are scarce relative to dimension and counts are strongly overdispersed.
\texttt{GLMNet(Poisson)} becomes the better practical choice as $N/D$ grows and the data become closer to the Poisson setting.
For network inference, PLNNetwork is stronger on broad undirected community graphs, whereas \texttt{GLMNet(Poisson)} is better matched to local or directional effects.
Although PLN is substantially more expensive computationally, its observed costs remain practical on the benchmark sizes studied here.
These results give researchers concrete guidance for deciding when latent dependence modeling is likely to justify its added complexity and when a simpler penalized marginal model is preferable.

\paragraph{Key findings.}
The benchmark yields five main observations.
First, $N/D$ is the strongest predictor of the count-prediction winner.
Second, MAC is the strongest secondary signal: datasets with high inter-taxon dependence tend to favor PLN even after accounting for $N/D$.
Third, overdispersion provides additional discriminative power beyond $N/D$ and MAC, consistent with PLN's latent Gaussian layer being most beneficial when counts depart substantially from the Poisson assumption.
Fourth, the network results depend on the kind of biological truth being evaluated, with PLNNetwork favored by broad undirected structure and \texttt{GLMNet(Poisson)} favored by local or directional effects.
Fifth, PLN is substantially more expensive than \texttt{GLMNet(Poisson)}, but the observed costs remain manageable for contemporary microbiome benchmark sizes.

\paragraph{Future directions.}
Several directions follow from the limitations of this study.
First, the benchmark draws from a finite set of public datasets; broader coverage of environments such as soil, marine, and host-associated systems, as well as diverse sequencing protocols, would strengthen the generalisability of the conclusions.
Second, all algorithms received the same \texttt{log1p}-transformed input~\citep{doi:10.1128/mSystems.00016-19} to ensure that performance differences reflected model structure rather than preprocessing.
We also explored fitting PLN on raw counts, but the resulting comparisons were not stable: improvements were inconsistent across datasets and outliers became more prominent.
Different normalisation or offset strategies may still shift relative performance and remain an important direction for future work.
Third, the comparison is restricted to algorithms in the Poisson family; extending the protocol to negative-binomial and zero-inflated models would give a more complete picture of the count-modelling landscape~\citep{10.1038/s41598-024-76513-8,10.1101/538579,10.1093/bioadv/vbae167,Champion2024OneNet}.

\section*{Declarations}

\subsection*{Ethics approval and consent to participate}
Not applicable.

\subsection*{Consent for publication}
Not applicable.

\subsection*{Availability of data and materials}
The datasets analysed in this study are publicly available; accession details and references are provided in Table~S2.
The benchmark code is available at \url{https://github.com/EngineerDanny/pln_eval}.

\subsection*{Competing interests}
The authors declare that they have no competing interests.

\subsection*{Funding}
This work was supported by the National Science Foundation grant \#2125088 (Rules of Life Program).

\subsection*{Authors' contributions}
DA conceived the study, implemented the benchmark pipeline, performed the analyses, and drafted the manuscript.
JC advised on methodology, algorithm design, and manuscript revisions.
TDH advised on methodology, algorithm design, and manuscript revisions.
JM contributed to writing, review, and editing of the manuscript.
All authors read and approved the final manuscript.

\subsection*{Acknowledgments}
Computational work was performed on Northern Arizona University's Monsoon high-performance computing cluster~\citep{Buechler2025MonsoonIntro}.

\subsection*{Authors' information}
Not applicable.

\bibliographystyle{plainnat}
\bibliography{references}

@article{Love2014DESeq2,
  author  = {Love, Michael I. and Huber, Wolfgang and Anders, Simon},
  title   = {Moderated estimation of fold change and dispersion for {RNA}-seq data with {DESeq2}},
  journal = {Genome Biology},
  volume  = {15},
  number  = {12},
  pages   = {550},
  year    = {2014},
}

@article{Robinson2010edgeR,
  author  = {Robinson, Mark D. and McCarthy, Davis J. and Smyth, Gordon K.},
  title   = {{edgeR}: a {Bioconductor} package for differential expression analysis of digital gene expression data},
  journal = {Bioinformatics},
  volume  = {26},
  number  = {1},
  pages   = {139--140},
  year    = {2010},
}

@article{Alkanani2015,
  author  = {Alkanani, Aya K. and Hara, Naoaki and Gottlieb, Peter A. and Ir, Delnatte and Robertson, Charles E. and Wagner, Brian D. and Frank, Daniel N. and Zipris, Danny},
  title   = {Alterations in Intestinal Microbiota Correlate with Susceptibility to Type 1 Diabetes},
  journal = {Diabetes},
  year    = {2015},
  volume  = {64},
  number  = {10},
  pages   = {3510--3520}
}

@article{CastroNallar2015,
  author  = {Castro-Nallar, Eduardo and Shen, Yuwei and Freishtat, Robert J. and Pérez-Losada, Marcos and Manimaran, Satheesh and Liu, Guojun and Johnson, William E. and Crandall, Keith A.},
  title   = {Integrating Microbial and Host Transcriptomics to Characterize Asthma-Associated Microbial Communities},
  journal = {BMC Medical Genomics},
  year    = {2015},
  volume  = {8},
  pages   = {50}
}

@article{Chng2016,
  author  = {Chng, Kern Rei and Tay, Angeline S. L. and Li, Chaorong and Ng, Ai Hui Qian and Wang, Jie and Suri, Basanth and Matta, Sachin Ahluwalia and McGovern, Naoko and Janela, Bruno and Wong, Xin Ying and Sio, Yik Ying and Au, Brenda and Wilm, Andreas and Nagarajan, Niranjan and Oon, Hazel H. and Ginhoux, Florent},
  title   = {Whole Metagenome Profiling Reveals Skin Microbiome-Dependent Susceptibility to Atopic Dermatitis Flare},
  journal = {Nature Microbiology},
  year    = {2016},
  volume  = {1},
  pages   = {16106}
}

@article{Costea2017,
  author  = {Costea, Paul I. and Hildebrand, Falk and Arumugam, Manimozhiyan and Bäckhed, Fredrik and Blaser, Martin J. and Bushman, Frederic D. and de Vos, Willem M. and Ehrlich, S. Dusko and Fraser, Claire M. and Hattori, Masahira and Huttenhower, Curtis and Jeffery, Ian B. and Knights, Dan and Lewis, James D. and Ley, Ruth E. and Ochman, Howard and O'Toole, Paul W. and Quince, Christopher and Relman, David A. and Shanahan, Fergus and Sunagawa, Shinichi and Wang, Jun and Weinstock, George M. and Wu, Gary D. and Zeller, Georg and Enterotypes Study Group and Bork, Peer},
  title   = {Enterotypes in the Landscape of Gut Microbial Community Composition},
  journal = {Nature Microbiology},
  year    = {2017},
  volume  = {3},
  pages   = {8--16}
}

@article{HMP2012,
  author  = {{Human Microbiome Project Consortium}},
  title   = {Structure, Function and Diversity of the Healthy Human Microbiome},
  journal = {Nature},
  year    = {2012},
  volume  = {486},
  number  = {7402},
  pages   = {207--214}
}

@article{Lozupone2013,
  author  = {Lozupone, Catherine A. and Li, Meng and Campbell, Tara B. and Flores, Sonia C. and Linderman, Daniel and Gebert, Michael J. and Knight, Rob and Fontenot, Andrew P. and Palmer, B. E.},
  title   = {Alterations in the Gut Microbiota Associated with {HIV}-1 Infection},
  journal = {Cell Host \& Microbe},
  year    = {2013},
  volume  = {14},
  number  = {3},
  pages   = {329--339}
}

@article{MBQC2017,
  author  = {Sinha, Rashmi and Abu-Ali, Gina and Vogtmann, Emily and Fodor, Anthony A. and Ren, Bo and Amir, Amnon and Schwager, Eric and Crabtree, Jacquelyn and Ma, Shun and The Microbiome Quality Control Project Consortium and others},
  title   = {Assessment of Variation in Microbial Community Amplicon Sequencing by the Microbiome Quality Control ({MBQC}) Project Consortium},
  journal = {Nature Biotechnology},
  year    = {2017},
  volume  = {35},
  number  = {11},
  pages   = {1077--1086}
}

@article{Vogtmann2016CRC,
  author  = {Vogtmann, Emily and Hua, Xin and Zeller, Georg and Sunagawa, Shinichi and Voigt, Anita Y. and Hercog, Rholig and Goedert, James J. and Shi, Jianxin and Bork, Peer and Sinha, Rashmi},
  title   = {Colorectal Cancer and the Human Gut Microbiome: Reproducibility with Whole-Genome Shotgun Sequencing},
  journal = {PLOS ONE},
  year    = {2016},
  volume  = {11},
  number  = {5},
  pages   = {e0155362},
  doi     = {10.1371/journal.pone.0155362}
}

@article{Weiss2022OMM12,
  author  = {Weiss, Anna S. and Burrichter, Anna G. and Durai Raj, Abilash Chakravarthy and von Strempel, Alexandra and Meng, Chen and others},
  title   = {In Vitro Interaction Network of a Synthetic Gut Bacterial Community},
  journal = {The ISME Journal},
  year    = {2022},
  volume  = {16},
  number  = {4},
  pages   = {1095--1109},
  doi     = {10.1038/s41396-021-01153-z}
}

@article{Weiss2023Keystone,
  author  = {Weiss, Anna S. and Niedermeier, Lisa S. and von Strempel, Alexandra and Burrichter, Anna G. and Ring, Diana and others},
  title   = {Nutritional and Host Environments Determine Community Ecology and Keystone Species in a Synthetic Gut Bacterial Community},
  journal = {Nature Communications},
  year    = {2023},
  volume  = {14},
  pages   = {4780},
  doi     = {10.1038/s41467-023-40372-0}
}

@article{Clark2021Butyrate,
  author  = {Clark, Ryan L. and Connors, Bryce M. and Stevenson, David M. and Hromada, Susan E. and Hamilton, Joshua J. and others},
  title   = {Design of Synthetic Human Gut Microbiome Assembly and Butyrate Production},
  journal = {Nature Communications},
  year    = {2021},
  volume  = {12},
  pages   = {3254},
  doi     = {10.1038/s41467-021-22938-y}
}

@article{Gould2018HostFitness,
  author  = {Gould, Alison L. and Zhang, Vivian and Lamberti, Lisa and Jones, Eric W. and Obadia, Benjamin and others},
  title   = {Microbiome Interactions Shape Host Fitness},
  journal = {Proceedings of the National Academy of Sciences},
  year    = {2018},
  volume  = {115},
  number  = {51},
  pages   = {E11951--E11960},
  doi     = {10.1073/pnas.1809349115}
}

@article{Zhu2025PairInteraX,
  author  = {Zhu, Jiaying and Jiang, Min-Zhi and Chen, Xue and Li, Min and Wang, Yu-Lin and others},
  title   = {Systematic Pairwise Co-Cultures Uncover Predominant Negative Interactions Among Human Gut Bacteria},
  journal = {Microbiome},
  year    = {2025},
  volume  = {13},
  pages   = {161},
  doi     = {10.1186/s40168-025-02156-0}
}

@article{McDonald2018,
  author  = {McDonald, Daniel and Hyde, Embriette and Debelius, Justine W. and Morton, James T. and Gonzalez, Antonio and Ackermann, Gail and others},
  title   = {American Gut: an Open Platform for Citizen Science Microbiome Research},
  journal = {mSystems},
  year    = {2018},
  volume  = {3},
  number  = {3},
  pages   = {e00031-18}
}

@article{Mehta2018,
  author  = {Mehta, Rishika S. and Abu-Ali, Gina S. and Drew, David A. and Lloyd-Price, Jason and Subramanian, Ashwin and Lochhead, Pius and Joshi, Asad D. and Ivey, Kristina L. and Khalili, Hamed and Brown, Gerald T. and DuLong, Caroline and Song, Minmin and Nguyen, Lu and Mallick, Himel and Rimm, Eric B. and Izard, Jacques and Huttenhower, Curtis and Chan, Andrew T.},
  title   = {Stability of the Human Faecal Microbiome in a Cohort of Adult Men},
  journal = {Nature Microbiology},
  year    = {2018},
  volume  = {3},
  number  = {3},
  pages   = {347--355}
}

@article{Ross2015,
  author  = {Ross, M. Caroline and Muzny, Donna M. and McCormick, J. B. and Gibbs, Richard A. and Fisher-Hoch, Susan P. and Petrosino, Joseph F.},
  title   = {16S Gut Community of the Cameron County Hispanic Cohort},
  journal = {Microbiome},
  year    = {2015},
  volume  = {3},
  pages   = {7}
}

@article{Schubert2014,
  author  = {Schubert, Alyxandria M. and Rogers, Mary A. M. and Ring, Cathrin and Mogle, Jillian and Petrosino, Joseph P. and Young, Vincent B. and Aronoff, David M. and Schloss, Patrick D.},
  title   = {Microbiome Data Distinguish Patients with {Clostridium difficile} Infection and Non-{C. difficile}-Associated Diarrhea from Healthy Controls},
  journal = {mBio},
  year    = {2014},
  volume  = {5},
  number  = {3},
  pages   = {e01021-14}
}

@article{Shao2019,
  author  = {Shao, Ying and Forster, Sam C. and Tsaliki, Evangelia and Vervier, Koen and Strang, Anna and Simpson, Nigel and Kumar, Nitin and Stares, Marta D. and Rodger, Adam and Brocklehurst, Peter and Field, Nick and Lawley, Trevor D.},
  title   = {Stunted Microbiota and Opportunistic Pathogen Colonization in Caesarean-Section Birth},
  journal = {Nature},
  year    = {2019},
  volume  = {574},
  number  = {7776},
  pages   = {117--121}
}

@article{Vatanen2016,
  author  = {Vatanen, Tommi and Kostic, Aleksandar D. and d'Hennezel, Emmanuel and Siljander, Heli and Franzosa, Eric A. and Yassour, Moran and Kolde, Raivo and Vlamakis, Hera and Arthur, Timothy D. and Hämäläinen, Anu-Maaria and Peet, Vallo and Tillmann, Vallo and Uibo, Raul and Mokurov, Sergei and Dorshakova, Natalia and Ilonen, Jorma and Virtanen, Suvi M. and Szabo, Steven J. and Porter, James A. and Lahdesmaki, Harri and Huttenhower, Curtis and Gevers, Dirk and Cullen, Thomas W. and Knip, Mikael and Xavier, Ramnik J. and others},
  title   = {Variation in Microbiome {LPS} Immunogenicity Contributes to Autoimmunity in Humans},
  journal = {Cell},
  year    = {2016},
  volume  = {165},
  number  = {4},
  pages   = {842--853}
}

@article{Wang2012,
  author  = {Wang, Tao and Cai, Guang and Qiu, Yue and Fei, Nanyi and Zhang, Ming and Pang, Xueya and Jia, Wenyi and Cai, Sumei and Zhao, Liping},
  title   = {Structural Segregation of Gut Microbiota between Colorectal Cancer Patients and Healthy Volunteers},
  journal = {The ISME Journal},
  year    = {2012},
  volume  = {6},
  number  = {2},
  pages   = {320--329}
}

@article{Wong2013,
  author  = {Wong, Vincent W.-S. and Tse, Chun-Hung and Lam, Timothy T.-Y. and Wong, Grace L.-H. and Chim, Albert M.-L. and Chu, William C.-W. and Yeung, Dora K.-W. and Law, Po-Tien and Kwan, Ho-Sum and Yu, Jun and Sung, Joseph J.-Y. and Chan, Henry L.-Y.},
  title   = {Molecular Characterization of the Fecal Microbiota in Patients with Nonalcoholic Steatohepatitis: a Longitudinal Study},
  journal = {PLoS ONE},
  year    = {2013},
  volume  = {8},
  number  = {4},
  pages   = {e62885}
}

@article{Zeevi2015,
  author  = {Zeevi, David and Korem, Tal and Zmora, Niv and Israeli, David and Rothschild, David and Weinberger, Adina and Ben-Yacov, Or and Lador, Dor and Avnit-Sagi, Tali and Lotan-Pompan, Michal and Suez, Jotham and Mahdi, Juman and Matot, Eran and Malka, Gil and Kosower, Noam and Rein, Michael and Zilberman-Schapira, Gili and Dohnalová, Lenka and Pevsner-Fischer, Miri and Bikovsky, Ron and Halpern, Ziv and Elinav, Eran and Segal, Eran},
  title   = {Personalized Nutrition by Prediction of Glycemic Responses},
  journal = {Cell},
  year    = {2015},
  volume  = {163},
  number  = {5},
  pages   = {1079--1094}
}

@article{Zeller2014,
  author  = {Zeller, Georg and Tap, Julien and Voigt, Anita Y. and Sunagawa, Shinichi and Kultima, Jens Roat and Costea, Paul I. and Amiot, Aurélien and Böhm, Jutta and Brunetti, Francesco and Habermann, Nicolas and Hercog, Romuald and Koch, Moritz and Luciani, Andrea and Mende, David R. and Schneider, Maike A. and Schrotz-King, Petra and Tournigand, Christophe and Tran Van Nhieu, Jim and Yamada, Takuji and Zimmermann, Jörg and Benes, Vladimir and Kloor, Matthias and Ulrich, Alexis and von Knebel Doeberitz, Magnus and Sobhani, Iradj and Bork, Peer},
  title   = {Potential of Fecal Microbiota for Early-Stage Detection of Colorectal Cancer},
  journal = {Molecular Systems Biology},
  year    = {2014},
  volume  = {10},
  pages   = {766}
}

@article{Zhu2013,
  author  = {Zhu, Lefei and Baker, Stephen S. and Gill, Christopher and Liu, Weiying and Alkhouri, Naim and Baker, Robert D. and Gill, Sandeep R.},
  title   = {Characterization of Gut Microbiomes in Nonalcoholic Steatohepatitis ({NASH}) Patients: a Connection between Endogenous Alcohol and {NASH}},
  journal = {Hepatology},
  year    = {2013},
  volume  = {57},
  number  = {2},
  pages   = {601--609}
}

@article{Thompson2017,
  title={A communal catalogue reveals Earth's multiscale microbial diversity},
  author={Thompson, Luke R and Sanders, Jon G and McDonald, Daniel and Amir, Amnon and Ladau, Joshua and Locey, Kenneth J and Prill, Robert J and Tripathi, Anupriya and Gibbons, Sean M and Ackermann, Gail and others},
  journal={Nature},
  volume={551},
  number={7681},
  pages={457--463},
  year={2017},
  publisher={Nature Publishing Group},
  doi={10.1038/nature24621}
}

@article{Huttenhower2012,
  title={A framework for human microbiome research},
  author={{Human Microbiome Project Consortium}},
  journal={Nature},
  volume={486},
  number={7402},
  pages={215--221},
  year={2012},
  publisher={Nature Publishing Group},
  doi={10.1038/nature11209}
}

@article{Chiquet2021,
  author  = {Chiquet, Julien and Mariadassou, Mahendra and Robin, St{\'e}phane},
  title   = {The Poisson--Lognormal Model as a Versatile Framework for the Joint Analysis of Species Abundances},
  journal = {Frontiers in Ecology and Evolution},
  volume  = {9},
  pages   = {588292},
  year    = {2021},
  doi     = {10.3389/fevo.2021.588292},
  url     = {https://www.frontiersin.org/articles/10.3389/fevo.2021.588292}
}

@article{SUBEDI20251255,
title = {Multivariate Poisson lognormal distribution for modeling counts from modern biological data: An overview},
journal = {Computational and Structural Biotechnology Journal},
volume = {27},
pages = {1255-1264},
year = {2025},
issn = {2001-0370},
doi = {https://doi.org/10.1016/j.csbj.2025.03.017},
url = {https://www.sciencedirect.com/science/article/pii/S2001037025000856},
author = {Sanjeena Subedi and Utkarsh J. Dang}
}

@article{Xia2023,
  author  = {Xia, Yinglin},
  title   = {Statistical Normalization Methods in Microbiome Data with Application to Microbiome Cancer Research},
  journal = {Gut Microbes},
  volume  = {15},
  number  = {2},
  pages   = {2244139},
  year    = {2023},
  doi     = {10.1080/19490976.2023.2244139}
}

@article{Kodikara2022,
  author  = {Kodikara, Saritha and Ellul, Susan and L{\^{e}} Cao, Kim-Anh},
  title   = {Statistical Challenges in Longitudinal Microbiome Data Analysis},
  journal = {Briefings in Bioinformatics},
  volume  = {23},
  number  = {4},
  pages   = {bbac273},
  year    = {2022},
  doi     = {10.1093/bib/bbac273}
}

@article{Qian2024,
  author  = {Qian, Weicheng and Stanley, Kevin G. and Aziz, Zohaib and Aziz, Umair and Siciliano, Steven D.},
  title   = {SPLANG—A Synthetic Poisson-Lognormal-Based Abundance and Network Generative Model for Microbial Interaction Inference Algorithms},
  journal = {Scientific Reports},
  volume  = {14},
  pages   = {25099},
  year    = {2024},
  doi     = {10.1038/s41598-024-76513-8}
}

@article{Friedman2010,
  author  = {Jerome H. Friedman and Trevor Hastie and Robert Tibshirani},
  title   = {Regularization Paths for Generalized Linear Models via Coordinate Descent},
  journal = {Journal of Statistical Software},
  year    = {2010},
  volume  = {33},
  number  = {1},
  pages   = {1--22},
  doi     = {10.18637/jss.v033.i01}
}

@article{Kenney2021PoissonPCA,
  author  = {Kenney, Toby and Gu, Hong and Huang, Tianshu},
  title   = {Poisson PCA: Poisson Measurement Error Corrected PCA, with Application to Microbiome Data},
  journal = {Biometrics},
  volume  = {77},
  number  = {4},
  pages   = {1369--1384},
  year    = {2021},
  doi     = {10.1111/biom.13384}
}

@inproceedings{Xiao2022PLNet,
  author    = {Xiao, Feiyi and Tang, Junjie and Fang, Huaying and Xi, Ruibin},
  title     = {Estimating Graphical Models for Count Data with Applications to Single-Cell Gene Network},
  booktitle = {Advances in Neural Information Processing Systems 35 (NeurIPS 2022)},
  year      = {2022}
}

@article{Virta2023PoissonPCA,
  author  = {Virta, Joni and Artemiou, Andreas},
  title   = {Poisson PCA for Matrix Count Data},
  journal = {Pattern Recognition},
  volume  = {138},
  pages   = {109401},
  year    = {2023},
  doi     = {10.1016/j.patcog.2023.109401}
}

@article{Batardiere2025ZIPLN,
  author  = {Batardi{\`e}re, Bastien and Chiquet, Julien and Gindraud, Fran{\c{c}}ois and Mariadassou, Mahendra},
  title   = {Zero-Inflation in the Multivariate Poisson-Lognormal Family},
  journal = {Statistics and Computing},
  year    = {2025},
  doi     = {10.1007/s11222-025-10729-0},
  note    = {Early Access}
}

@article{Chaussard2025PLNTree,
  author  = {Chaussard, Alexandre and Bonnet, Anna and Gassiat, Elisabeth and Le Corff, Sylvain},
  title   = {Tree-based Variational Inference for Poisson Log-Normal Models},
  journal = {Statistics and Computing},
  year    = {2025},
  doi     = {10.1007/s11222-025-10668-w},
  note    = {Early Access}
}

@article{aitchison1989multivariate,
  author  = {Aitchison, J. and Ho, C. H.},
  title   = {The multivariate Poisson-log normal distribution},
  journal = {Biometrika},
  year    = {1989},
  volume  = {76},
  number  = {4},
  pages   = {643--653},
  doi     = {10.1093/biomet/76.4.643}
}

@inproceedings{chiquet2019variational,
  author    = {Chiquet, Julien and Robin, Stephane and Mariadassou, Mahendra},
  title     = {Variational Inference for Sparse Network Reconstruction from Count Data},
  booktitle = {Proceedings of the 36th International Conference on Machine Learning},
  series    = {Proceedings of Machine Learning Research},
  volume    = {97},
  pages     = {1162--1171},
  year      = {2019},
  publisher = {PMLR},
  url       = {https://proceedings.mlr.press/v97/chiquet19a.html}
}

@article{zhang2017nbmm,
  title   = {Negative Binomial Mixed Models for Analyzing Microbiome Count Data},
  author  = {Zhang, Xinyan and Mallick, Himel and Tang, Zaixiang and Zhang, Lei and Cui, Xiangqin and Benson, Andrew K. and Yi, Nengjun},
  journal = {BMC Bioinformatics},
  volume  = {18},
  number  = {1},
  pages   = {4},
  year    = {2017},
  doi     = {10.1186/s12859-016-1441-7}
}

@article{mcmurdie2014waste,
  title   = {Waste Not, Want Not: Why Rarefying Microbiome Data Is Inadmissible},
  author  = {McMurdie, Paul J. and Holmes, Susan},
  journal = {PLOS Computational Biology},
  volume  = {10},
  number  = {4},
  pages   = {e1003531},
  year    = {2014},
  doi     = {10.1371/journal.pcbi.1003531}
}

@article{booeshaghi2021log1p,
  author  = {Booeshaghi, A. Sina and Pachter, Lior},
  title   = {Normalization of single-cell RNA-seq counts by log(x + 1) or log(1 + x)},
  journal = {Bioinformatics},
  year    = {2021},
  volume  = {37},
  number  = {15},
  pages   = {2223--2224},
  doi     = {10.1093/bioinformatics/btab085}
}

@article{doi:10.1128/mSystems.00016-19,
author = {Cameron Martino and James T. Morton and Clarisse A. Marotz and Luke R. Thompson and Anupriya Tripathi and Rob Knight and Karsten Zengler},
title = {A Novel Sparse Compositional Technique Reveals Microbial Perturbations},
journal = {mSystems},
volume = {4},
number = {1},
pages = {e00016-19},
year = {2019},
doi = {10.1128/mSystems.00016-19},
URL = {https://journals.asm.org/doi/abs/10.1128/mSystems.00016-19},
eprint = {https://journals.asm.org/doi/pdf/10.1128/mSystems.00016-19}
}

@article{10.1038/s41598-024-76513-8,
    author = {Wang, Yue and Xu, Jie and Zhang, Wei and Chen, Kun},
    title = "{SPLANG: A sparse learning approach for inferring gene regulatory networks}",
    journal = {Scientific Reports},
    volume = {14},
    pages = {1674},
    year = {2024},
    doi = {10.1038/s41598-024-76513-8}
}

@article{10.1101/538579,
    author = {Zhang, Chen and Liu, Yang and Wang, Jing and Zhao, Hong},
    title = "{MAGMA: Inference of Sparse Microbial Association Networks}",
    journal = {bioRxiv},
    year = {2019},
    doi = {10.1101/538579}
}

@article{10.1093/bioadv/vbae167,
    author = {Li, Hongzhe and Ma, Zongming and Liu, Han},
    title = "{MicroNet-MIMRF: Microbial network inference via mixed integer matrix recovery framework}",
    journal = {Bioinformatics Advances},
    volume = {4},
    number = {1},
    year = {2024},
    doi = {10.1093/bioadv/vbae167}
}

@misc{Buechler2025MonsoonIntro,
  author       = {Buechler, Jason},
  title        = {Intro to Monsoon and Slurm},
  year         = {2025},
  howpublished = {\url{https://rcdata.nau.edu/hpcpub/workshops/odintro.pdf}},
  urldate      = {2025-12-12}
}

@article{Champion2024OneNet,
  author  = {Champion, Camille and Momal, Rapha{\"e}lle and Le Chatelier, Emmanuelle and Sola, Mathilde and Mariadassou, Mahendra and Berland, Magali},
  title   = {OneNet---One network to rule them all: Consensus network inference from microbiome data},
  journal = {PLOS Computational Biology},
  year    = {2024},
  volume  = {20},
  number  = {12},
  pages   = {e1012627},
  doi     = {10.1371/journal.pcbi.1012627}
}

@article{Meinshausen2006,
  author  = {Meinshausen, Nicolai and B{\"u}hlmann, Peter},
  title   = {High-Dimensional Graphs and Variable Selection with the Lasso},
  journal = {The Annals of Statistics},
  year    = {2006},
  volume  = {34},
  number  = {3},
  pages   = {1436--1462}
}

@article{Ghaeli2025,
  author  = {Ghaeli, Zahra and Aghdam, Rosa and Eslahchi, Changiz},
  title   = {Evaluating Microbial Network Inference Methods: Moving Beyond Synthetic Data with Reproducibility-Driven Benchmarks},
  journal = {bioRxiv},
  year    = {2025},
  doi     = {10.1101/2025.07.05.663212},
  note    = {Preprint}
}

@article{Kurtz2015SPIECEASI,
  author  = {Kurtz, Zachary D and M{\"u}ller, Christian L and Miraldi, Emily R and Littman, Dan R and Blaser, Martin J and Bonneau, Richard A},
  title   = {Sparse and compositionally robust inference of microbial ecological networks},
  journal = {PLOS Computational Biology},
  volume  = {11},
  number  = {5},
  pages   = {e1004226},
  year    = {2015},
  doi     = {10.1371/journal.pcbi.1004226}
}

@article{Friedman2012SparCC,
  author  = {Friedman, Jonathan and Alm, Eric J},
  title   = {Inferring correlation networks from genomic survey data},
  journal = {PLOS Computational Biology},
  volume  = {8},
  number  = {9},
  pages   = {e1002687},
  year    = {2012},
  doi     = {10.1371/journal.pcbi.1002687}
}

@article{Gneiting2007ProperScoring,
  author  = {Gneiting, Tilmann and Raftery, Adrian E.},
  title   = {Strictly Proper Scoring Rules, Prediction, and Estimation},
  journal = {Journal of the American Statistical Association},
  volume  = {102},
  number  = {477},
  pages   = {359--378},
  year    = {2007},
  doi     = {10.1198/016214506000001437}
}

\clearpage
\appendix
\section*{Supplementary Material}

\section*{Appendix S1. LOTO-CV: Formal Definition and Theoretical Guarantees}

\subsection*{S1.1\quad Notation}

Let $Y \in \mathbb{Z}_{\geq 0}^{N \times D}$ denote the observed count matrix with $N$ samples and $D$ taxa.
Write $Y_{\cdot j} \in \mathbb{Z}_{\geq 0}^{N}$ for the $j$-th column (the count vector of taxon $j$ across all samples) and $Y_{\cdot,-j} \in \mathbb{Z}_{\geq 0}^{N \times (D-1)}$ for the matrix of all remaining taxa.
Let $\mathcal{I}_1,\ldots,\mathcal{I}_K$ denote a partition of the samples into $K$ folds.
Let $\mathcal{A}$ denote a count regression algorithm.
For each target taxon $j$ and sample fold $k$, the algorithm is trained on the samples outside $\mathcal{I}_k$ using taxon $j$ as the response and the remaining taxa as predictors.
Let $\hat{\mu}^{(-k,j)}_{ij} \in \mathbb{R}_{>0}$ denote the predicted mean for held-out sample $i \in \mathcal{I}_k$ and target taxon $j$.
The mean Poisson deviance between observed counts $y \in \mathbb{Z}_{\geq 0}^{m}$ and predicted means $\hat{\mu} \in \mathbb{R}_{>0}^{m}$ is
\[
  \mathcal{D}(y,\, \hat{\mu})
  \;=\;
  \frac{2}{m} \sum_{i=1}^{m} \left( y_i \log \frac{y_i}{\hat{\mu}_i} - (y_i - \hat{\mu}_i) \right),
\]
with the convention $0 \log 0 = 0$.

\subsection*{S1.2\quad The LOTO-CV Estimator}

The estimator reported in the paper is
\[
  \widehat{\mathcal{D}}_{\mathrm{LOTO}\text{-}\mathrm{CV}}(\mathcal{A})
  \;=\;
  \frac{1}{DK} \sum_{j=1}^{D}\sum_{k=1}^{K}
  \mathcal{D}\!\left(Y_{\mathcal{I}_k j},\; \hat{\mu}^{(-k,j)}_{\mathcal{A}}\right),
\]
where $Y_{\mathcal{I}_k j}$ is the vector of observed counts for target taxon $j$ on the held-out samples in fold $k$.
Each taxon serves once as the prediction target, and each sample serves once as a held-out observation within each taxon-specific regression problem.

\subsection*{S1.3\quad Poisson Deviance as a Proper Scoring Rule}

Proper scoring rules are loss functions minimised in expectation by the true conditional distribution \citep{Gneiting2007ProperScoring}.
We establish properness of the Poisson deviance via the Bregman divergence representation.
Let $\phi: \mathbb{R}_{\geq 0} \to \mathbb{R}$ be the strictly convex function $\phi(\mu) = \mu \log \mu - \mu$, extended continuously at zero so that $\phi(0)=0$.
The associated Bregman divergence is
\[
  B_{\phi}(y,\,\mu)
  \;=\;
  \phi(y) - \phi(\mu) - \phi'(\mu)(y - \mu)
  \;=\;
  y \log\frac{y}{\mu} - (y - \mu),
\]
so that for vectors of length $m$,
\[
  \mathcal{D}(y,\mu) = \frac{2}{m} \sum_i B_\phi(y_i, \mu_i).
\]

\begin{proposition}[Properness of Poisson deviance]
Let $Y$ be a non-negative integer-valued random variable with finite mean $\mu^* = \mathbb{E}[Y]$.
Then for all $\mu > 0$,
\[
  \mathbb{E}\!\left[B_\phi(Y,\,\mu^*)\right]
  \;\leq\;
  \mathbb{E}\!\left[B_\phi(Y,\,\mu)\right],
\]
with equality if and only if $\mu = \mu^*$.
\end{proposition}

\begin{proof}
By linearity of expectation and the definition of Bregman divergence,
\[
  \mathbb{E}\!\left[B_\phi(Y,\,\mu)\right]
  \;=\;
  \mathbb{E}[\phi(Y)] - \phi(\mu) - \phi'(\mu)\bigl(\mu^* - \mu\bigr).
\]
Similarly,
\[
  \mathbb{E}\!\left[B_\phi(Y,\,\mu^*)\right]
  \;=\;
  \mathbb{E}[\phi(Y)] - \phi(\mu^*).
\]
Subtracting and using the definition of $B_\phi$,
\[
  \mathbb{E}\!\left[B_\phi(Y,\,\mu)\right]
  - \mathbb{E}\!\left[B_\phi(Y,\,\mu^*)\right]
  \;=\;
  \phi(\mu^*) - \phi(\mu) - \phi'(\mu)\bigl(\mu^* - \mu\bigr)
  \;=\;
  B_\phi(\mu^*,\,\mu)
  \;\geq\; 0,
\]
where the inequality holds because $\phi$ is convex, so $B_\phi \geq 0$ everywhere, and equality $B_\phi(\mu^*, \mu) = 0$ holds if and only if $\mu = \mu^*$ by strict convexity of $\phi$.
\end{proof}

\noindent The Proposition states that no predicted mean $\mu \neq \mu^*$ can achieve lower expected deviance than the true conditional mean.
A model with lower held-out Poisson deviance therefore predicts better in a statistically meaningful sense, not merely in a sample-specific sense.
This is the guarantee that in-sample criteria such as ELBO and BIC do not provide.

\subsection*{S1.4\quad Population Target and the Exchangeability Assumption}

The population quantity that $\widehat{\mathcal{D}}_{\mathrm{LOTO}\text{-}\mathrm{CV}}$ estimates is
\[
  R(\mathcal{A})
  \;=\;
  \mathbb{E}\!\left[\,\mathcal{D}\!\left(Y_{\mathcal{I}_k j},\; \hat{\mu}^{(-k,j)}_{\mathcal{A}}\right)\right],
\]
where the expectation is over new realisations of the count matrix together with the random choice of target taxon $j$ and sample fold $k$.
Under \emph{taxon exchangeability} --- the assumption that the $D$ columns of $Y$ are exchangeable draws from the same marginal distribution --- $\widehat{\mathcal{D}}_{\mathrm{LOTO}\text{-}\mathrm{CV}}$ is an approximately unbiased estimator of $R(\mathcal{A})$, and $\mathcal{A}$ is selected by minimising this estimate.

In real microbiome data, taxa are not exchangeable: they span different phyla, have different mean abundances, and exhibit different sparsity profiles.
Taxon exchangeability is therefore an idealisation.
However, any bias this introduces applies equally to all algorithms being compared, so it does not affect the validity of the relative comparison between PLN and \texttt{GLMNet(Poisson)}.

\subsection*{S1.5\quad Outer $K$-fold Approximation}

For computational tractability, the paper uses $K = 3$ outer folds over samples.
For each target taxon $j$, the algorithm is trained on samples outside $\mathcal{I}_k$ and evaluated on held-out samples in $\mathcal{I}_k$.
This produces one held-out score for each target-taxon and sample-fold pair.
The reported benchmark score is the average of these $D \times K$ scores.

\clearpage
\section*{Table S2. Count-prediction dataset inventory}
Datasets are sorted by $N/D$.
$^\dagger$American Gut~1 is an outlier in the count-prediction benchmark (see Table~S4).
\footnotesize
\begin{center}
\begin{tabular}{>{\raggedright\arraybackslash}p{0.19\textwidth}>{\raggedright\arraybackslash}p{0.24\textwidth}>{\raggedright\arraybackslash}p{0.20\textwidth}rrr}
\toprule
Dataset & Reference & Human system & $N$ & $D$ & Sparsity (\%) \\
\midrule
Castro-Nallar 2015 & \citealt{CastroNallar2015} & oral & 32 & 107 & 60.28 \\
ChngKR 2016 & \citealt{Chng2016} & skin & 78 & 211 & 80.66 \\
HMP 2012 & \citealt{HMP2012} & nasal & 91 & 222 & 92.72 \\
Zeller CRC & \citealt{Zeller2014} & stool CRC & 129 & 257 & 58.29 \\
Ross obesity & \citealt{Ross2015} & gut obesity & 63 & 123 & 67.05 \\
Lozupone HIV & \citealt{Lozupone2013} & gut HIV & 55 & 62 & 65.75 \\
Baker NASH/obesity & \citealt{Zhu2013} & gut NASH/obesity & 64 & 55 & 56.85 \\
Zhao CRC & \citealt{Wang2012} & stool CRC & 102 & 78 & 77.48 \\
CosteaPI 2017 & \citealt{Costea2017} & stool enterotypes & 279 & 179 & 71.73 \\
American Gut 2 & \citealt{McDonald2018} & gut citizen science & 294 & 138 & 34.48 \\
Chan NASH & \citealt{Wong2013} & gut NASH & 54 & 24 & 63.12 \\
American Gut 1$^\dagger$ & \citealt{McDonald2018} & gut citizen science & 289 & 127 & 30.64 \\
Alkanani T1D & \citealt{Alkanani2015} & gut T1D & 112 & 27 & 66.34 \\
Schubert CDI & \citealt{Schubert2014} & stool CDI & 347 & 79 & 79.35 \\
MehtaRS 2018 & \citealt{Mehta2018} & gut longitudinal & 928 & 183 & 75.72 \\
ShaoY 2019 & \citealt{Shao2019} & infant gut & 1644 & 244 & 93.01 \\
Zeevi obesity & \citealt{Zeevi2015} & gut glycemic response & 820 & 74 & 74.99 \\
HMP16S v13 & \citealt{Huttenhower2012} & multi-site 16S & 2909 & 159 & 87.20 \\
Diabimmune Karelia & \citealt{Vatanen2016} & infant gut & 1584 & 55 & 53.36 \\
MBQC integrated OTUs & \citealt{MBQC2017} & stool benchmark & 18270 & 235 & 89.46 \\
\bottomrule
\end{tabular}
\end{center}
\normalsize

\section*{Table S3. Network-inference dataset metadata}
\scriptsize
\setlength{\tabcolsep}{2.5pt}
\begin{center}
\begin{tabular}{>{\raggedright\arraybackslash}p{0.14\textwidth}>{\raggedright\arraybackslash}p{0.15\textwidth}>{\raggedright\arraybackslash}p{0.14\textwidth}>{\raggedright\arraybackslash}p{0.21\textwidth}rrr}
\toprule
Dataset & Reference & Human system & Truth type & $N$ & $D$ & Sparsity (\%) \\
\midrule
OMM12 & \citealt{Weiss2022OMM12} & defined human gut bacterial community & experimental coculture and dropout interactions & 109 & 12 & 21.25 \\
OMM12 keystone 2023 & \citealt{Weiss2023Keystone} & defined human gut bacterial community & experimental dropout-derived abundance-ratio effects & 228 & 11 & 31.50 \\
PairInteraX & \citealt{Zhu2025PairInteraX} & human gut species panel & experimental pairwise co-culture interaction labels & 169 & 96 & 25.86 \\
Butyrate assembly 2021 & \citealt{Clark2021Butyrate} & defined human gut community & pair-vs-singleton abundance shifts & 537 & 26 & 66.19 \\
Host fitness 2018 & \citealt{Gould2018HostFitness} & \textit{Drosophila} gut community & pair-vs-singleton CFU shifts & 1449 & 5 & 54.56 \\
\bottomrule
\end{tabular}
\end{center}
\setlength{\tabcolsep}{6pt}
\normalsize

\section*{Table S4. 20-dataset count-prediction results}
Datasets are sorted by $N/D$.
Blue rows are the four strongest PLN-favorable contrasts; orange rows are the four strongest GLMNet-favorable contrasts.
GLMNet and PLN entries report held-out Poisson deviance as mean $\pm$ SE.
$p$-values are from a paired Wilcoxon signed-rank test on per-fold deviance differences; $q$-values are Benjamini--Hochberg adjusted.
$^\dagger$American Gut~1 is an outlier: GLMNet wins despite low $N/D$ and relatively high MAC.
\footnotesize
\setlength{\tabcolsep}{2.7pt}
\setlength{\LTleft}{0pt}
\setlength{\LTright}{0pt}
\begin{longtable}{p{0.20\textwidth}rrp{0.155\textwidth}p{0.155\textwidth}rrrp{0.08\textwidth}}
\toprule
Dataset & $N/D$ & MAC & GLMNet & PLN & $\Delta$ & $p$-value & $q$-value & Winner \\
\midrule
\endfirsthead
\toprule
Dataset & $N/D$ & MAC & GLMNet & PLN & $\Delta$ & $p$-value & $q$-value & Winner \\
\midrule
\endhead
Castro-Nallar 2015 & 0.30 & 0.167 & 0.4857 $\pm$ 0.0695 & 0.3538 $\pm$ 0.0248 & +27.2\% & 0.020 & 0.024 & PLN \\
ChngKR 2016 & 0.37 & 0.092 & 0.2139 $\pm$ 0.0152 & 0.1793 $\pm$ 0.0104 & +16.2\% & 5.3e-4 & 8.8e-4 & PLN \\
\rowcolor{blue!8} HMP 2012 & 0.41 & 0.118 & 0.3365 $\pm$ 0.0433 & 0.2075 $\pm$ 0.0214 & +38.3\% & 2.9e-4 & 5.3e-4 & PLN \\
Zeller CRC & 0.50 & 0.090 & 1.0870 $\pm$ 0.0559 & 0.9439 $\pm$ 0.0448 & +13.2\% & 1.4e-4 & 3.4e-4 & PLN \\
\rowcolor{blue!8} Ross obesity & 0.51 & 0.130 & 1.2938 $\pm$ 0.0885 & 0.9046 $\pm$ 0.0422 & +30.1\% & 1.8e-6 & 5.0e-6 & PLN \\
\rowcolor{blue!8} Lozupone HIV & 0.89 & 0.163 & 1.3591 $\pm$ 0.1128 & 0.9587 $\pm$ 0.0824 & +29.5\% & 2.0e-4 & 4.1e-4 & PLN \\
\rowcolor{blue!8} Baker NASH/obesity & 1.16 & 0.156 & 1.3292 $\pm$ 0.1412 & 0.9337 $\pm$ 0.0571 & +29.8\% & 0.001 & 0.002 & PLN \\
Zhao CRC & 1.31 & 0.091 & 0.7487 $\pm$ 0.0498 & 0.6205 $\pm$ 0.0295 & +17.1\% & 0.002 & 0.002 & PLN \\
\rowcolor{orange!10} CosteaPI 2017 & 1.56 & 0.055 & 0.1974 $\pm$ 0.0144 & 0.2241 $\pm$ 0.0152 & $-$13.5\% & 3.9e-13 & 1.6e-12 & GLMNet \\
American Gut 2 & 2.13 & 0.237 & 1.3108 $\pm$ 0.0321 & 1.1553 $\pm$ 0.0284 & +11.9\% & 3.5e-31 & 3.5e-30 & PLN \\
Chan NASH & 2.25 & 0.112 & 1.5787 $\pm$ 0.1904 & 1.2359 $\pm$ 0.1291 & +21.7\% & 0.035 & 0.038 & PLN \\
American Gut 1$^\dagger$ & 2.28 & 0.182 & 1.2049 $\pm$ 0.0369 & 1.2761 $\pm$ 0.0330 & $-$5.9\% & 9.4e-13 & 3.1e-12 & GLMNet \\
Alkanani T1D & 4.15 & 0.231 & 1.3291 $\pm$ 0.1450 & 1.0756 $\pm$ 0.1167 & +19.1\% & 0.004 & 0.005 & PLN \\
Schubert CDI & 4.39 & 0.074 & 0.7711 $\pm$ 0.0476 & 0.6905 $\pm$ 0.0449 & +10.5\% & 1.5e-4 & 3.4e-4 & PLN \\
\rowcolor{orange!10} MehtaRS 2018 & 5.07 & 0.042 & 0.1786 $\pm$ 0.0113 & 0.2225 $\pm$ 0.0140 & $-$24.6\% & 1.6e-23 & 8.0e-23 & GLMNet \\
ShaoY 2019 & 6.74 & 0.033 & 0.1571 $\pm$ 0.0124 & 0.1647 $\pm$ 0.0136 & $-$4.8\% & 0.002 & 0.003 & GLMNet \\
Zeevi obesity & 11.08 & 0.110 & 0.5570 $\pm$ 0.0334 & 0.5628 $\pm$ 0.0326 & $-$1.1\% & 0.424 & 0.424 & GLMNet \\
HMP16S v13 & 18.30 & 0.058 & 0.3307 $\pm$ 0.0180 & 0.3347 $\pm$ 0.0191 & $-$1.2\% & 0.284 & 0.299 & GLMNet \\
\rowcolor{orange!10} Diabimmune Karelia & 28.80 & 0.127 & 1.1932 $\pm$ 0.0543 & 1.2794 $\pm$ 0.0564 & $-$7.2\% & 5.6e-26 & 3.7e-25 & GLMNet \\
\rowcolor{orange!10} MBQC integrated OTUs & 77.74 & 0.041 & 0.2337 $\pm$ 0.0127 & 0.2996 $\pm$ 0.0165 & $-$28.2\% & 2.6e-34 & 5.3e-33 & GLMNet \\
\bottomrule
\end{longtable}
\setlength{\LTleft}{0pt}
\setlength{\LTright}{0pt}
\normalsize

\section*{Figure S2. Host fitness 2018 network comparison}
\begin{center}
\includegraphics[width=\textwidth]{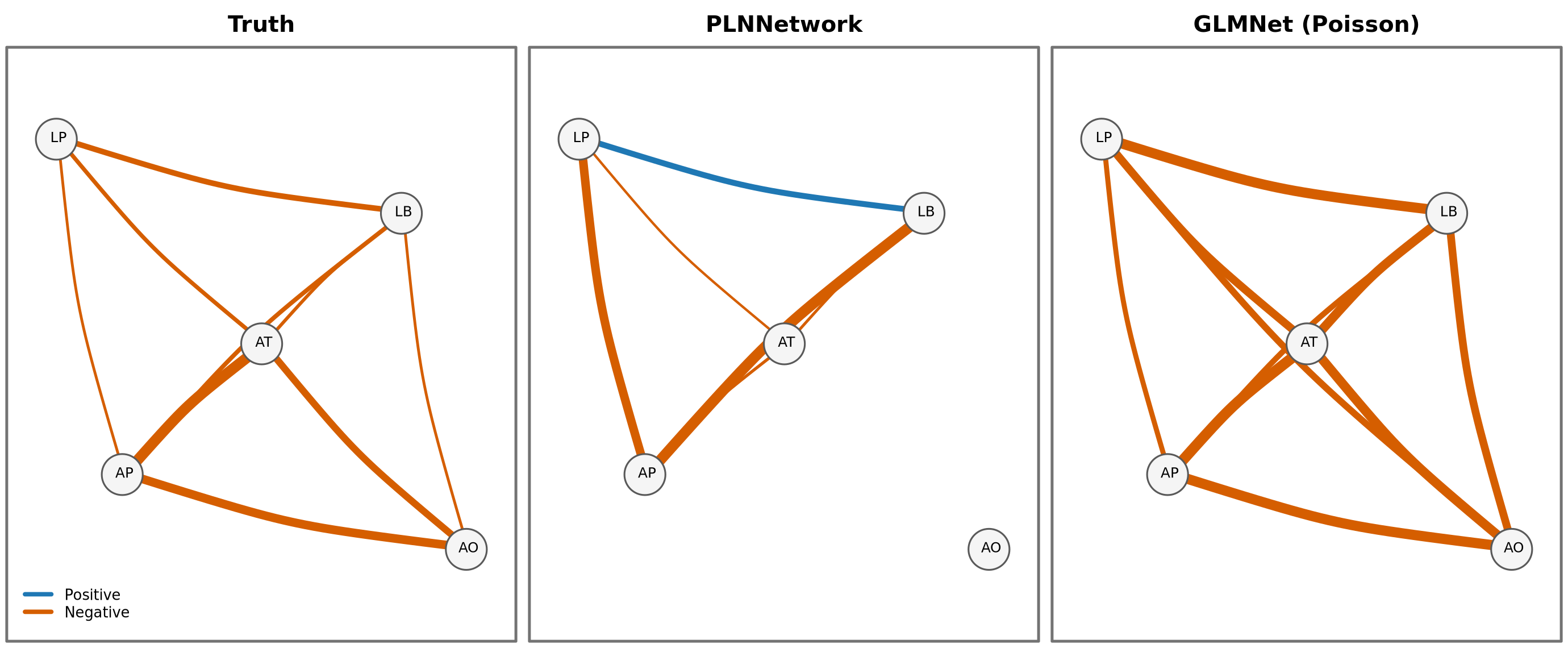}
\end{center}
Ground-truth and inferred networks for the host fitness 2018 dataset~\citep{Gould2018HostFitness}, a five-species \textit{Drosophila} gut community.
Each panel shows the experimental ground truth (left), the PLNNetwork prediction (centre), and the \texttt{GLMNet(Poisson)} prediction (right).
Node labels are species abbreviations: LP (\textit{Lactobacillus plantarum}), LB (\textit{Lactobacillus brevis}), AP (\textit{Acetobacter pasteurianus}), AT (\textit{Acetobacter tropicalis}), AO (\textit{Acetobacter orientalis}).
Edge colour indicates interaction sign (blue: positive, orange: negative).
Ground-truth edges are derived from pair-vs-singleton colony-forming unit shifts~\citep{Gould2018HostFitness}.

\end{document}